\documentclass[journal,compsoc,onecolumn]{IEEEtran}
\usepackage[font=footnotesize,labelfont=bf]{caption}
\usepackage{multirow}
\usepackage{bm}
\usepackage{caption}
\usepackage[table,xcdraw]{xcolor}
\usepackage[normalem]{ulem}
\usepackage[mathscr]{eucal}
\usepackage{adjustbox}
\useunder{\uline}{\ul}{}
\usepackage{algorithm}
\usepackage{algorithmic}
\usepackage{amsmath}
\usepackage[utf8]{inputenc}
\usepackage[english]{babel}
\usepackage{amsthm}
\theoremstyle{definition}

\newtheorem{theorem}{Theorem}

\newtheorem{lemma}{Lemma}

\usepackage{amsfonts}
\usepackage{amssymb}
\usepackage{caption}
\usepackage{dsfont}
\usepackage{graphicx}\graphicspath{ {./Pictures/} }
\usepackage[algo2e]{algorithm2e} 
\usepackage{amsmath}
\DeclareMathOperator*{\argmin}{argmin}
\usepackage{lipsum}
\usepackage{hyperref}
\usepackage{mathtools}

\usepackage{booktabs, multirow}
\usepackage{xcolor,colortbl} 
\usepackage{changepage,threeparttable} 

\begin{document}

\title{Hyperbolic Fuzzy C-Means with Adaptive Weight-based Filtering for Efficient Clustering}

\author{Swagato Das $^\ast$, Arghya Pratihar $^\ast$, Swagatam Das

\thanks{
$^\ast$ These authors contributed equally to this work. \\
Indian Statistical Institute, Kolkata, India. \\ 
E-mail:  swagato.isi2227@gmail.com,
arghyapratihar24@gmail.com,
swagatam.das@isical.ac.in.\\}
}

\maketitle
\begin{abstract}
Clustering algorithms play a pivotal role in unsupervised learning by identifying and grouping similar objects based on shared characteristics. Although traditional clustering techniques, such as hard and fuzzy center-based clustering, have been widely used, they struggle with complex, high-dimensional, and non-Euclidean datasets. In particular, the fuzzy $C$-Means (FCM) algorithm, despite its efficiency and popularity, exhibits notable limitations in non-Euclidean spaces. Euclidean spaces assume linear separability and uniform distance scaling, limiting their effectiveness in capturing complex, hierarchical, or non-Euclidean structures in fuzzy clustering. To overcome these challenges, we introduce Filtration-based Hyperbolic Fuzzy $C$-Means (HypeFCM), a novel clustering algorithm tailored for better representation of data relationships in non-Euclidean spaces. HypeFCM integrates the principles of fuzzy clustering with hyperbolic geometry and employs a weight-based filtering mechanism to improve performance. The algorithm initializes weights using a Dirichlet distribution and iteratively refines cluster centroids and membership assignments based on a hyperbolic metric in the Poincar\'e Disc model. Extensive experimental evaluations on $6$ synthetic and $12$ real-world datasets demonstrate that HypeFCM significantly outperforms conventional fuzzy clustering methods in non-Euclidean settings, underscoring its robustness and effectiveness.
\end{abstract}


\section{Introduction}
Clustering is an unsupervised learning technique used to group objects based on similarities, organizing similar items into clusters to reveal underlying patterns or structures. Clustering is a powerful tool for analyzing complex datasets by identifying meaningful groupings without predefined labels. A wide range of clustering algorithms has been proposed, broadly distinguished into hard and soft clustering methods. While hard clustering assigns each data point to a single cluster, soft clustering allows overlaps. Hard clustering methods include center-based, hierarchical-based, distribution-based, and density-based techniques. Center-based methods, such as $k$-means \cite{kmeans}, $k$-medoids \cite{k-medoids}, $k$-harmonic means \cite{kharmonic}, spectral clustering \cite{spectral}, kernel $k$-means \cite{kernel} measure similarity by proximity to cluster centers. Hierarchical-based methods, including hierarchical clustering \cite{hierarchical}, agglomerative clustering \cite{agglomerative} assume stronger similarities between closer data points. Distribution-based methods, such as Expectation Maximization (EM) for Gaussian Mixture Models \cite{gaussian}, robust EM for Gaussian Mixture Models \cite{robustgmm}, cluster the data based on probability distributions. Density-based methods, like DBSCAN \cite{dbscan}, HDBSCAN \cite{hdbscan}, Mean-Shift \cite{mean} identify clusters by analyzing data density in feature space.

Soft clustering allows data points to be associated with multiple clusters with varying degrees of membership. Cluster labels are based on the highest membership value. Notable techniques include possibilistic clustering \cite{pcm} and fuzzy clustering \cite{fuzzy}, with fuzzy $c$-means (FCM) \cite{fcm}, fuzzy density peaks clustering \cite{fuzzydensity}, centroid auto-fused hierarchical FCM \cite{centroidfuzzy}, robust FCM \cite{robust}.

FCM is widely used for its efficiency and simplicity, yet it struggles with complex, high-dimensional, and non-Euclidean datasets. To mitigate these limitations, several variants have been introduced, incorporating improved objective functions and constraints, such as adaptive FCM \cite{adaptive}, generalized FCM \cite{generalized}, fuzzy weighted $c$-means \cite{weightedfuzzy}, and generalized FCM with improved fuzzy partitioning \cite{generalizedfuzzy}. Kernel-based approaches like kernel FCM (KFCM) \cite{kernel} and constrained models, including agglomerative fuzzy $k$-means (AFKM) \cite{afkm}, robust self-sparse fuzzy clustering (RSSFCA) \cite{robust}, robust and sparse fuzzy $k$-means (RSFKM) \cite{rsfkm}, possibilistic FCM (PFCM) \cite{pfcm}, principal component analysis-embedded FCM (P\_SFCM) \cite{psfcm} as well as hyperbolic extensions such as hyperbolic smoothing-based fuzzy clustering (HSFC) \cite{hsfc}, Integration of hyperbolic tangent and Gaussian kernels for FCM (HGFCM) \cite{hgfcm}, have also been explored. However, despite these modifications, most of these methods remain fundamentally limited in non-Euclidean spaces, as they can still partially rely on Euclidean assumptions that fail to capture the inherent geometric complexity and hierarchical structure of such data. Consequently, these approaches often exhibit suboptimal clustering performance, reinforcing the need for a more robust, geometry-aware solution. Recently, Wu and Pan proposed FCPFM \cite{fcpfmc}, a fuzzy clustering algorithm in hyperbolic space that utilizes Fréchet mean for centroid updates, and introduced a kernelized variant of it. 

In this paper, we propose Filtration-based \textbf{Hype}rbolic \textbf{F}uzzy $\mathbf{C}$-\textbf{M}eans or HypeFCM, a novel fuzzy clustering algorithm for non-Euclidean spaces. Our approach not only integrates fuzzy clustering principles with hyperbolic geometry but also employs a weight-based filtering mechanism to improve cluster assignments. A defining strength of our algorithm lies in ensuring that the centroid update remains inherently constrained within the geometry of the Poincaré ball, thereby preserving the integrity of hyperbolic space throughout the clustering process. The algorithm initializes weights using a Dirichlet distribution, iteratively updates centroids and membership weights, and computes distances using the metric of the Poincar\'e Disc model. The Poincaré Disc metric naturally preserves hierarchical structures due to the exponential distance scaling and the negative curvature of hyperbolic space. In contrast to Euclidean space, where distances grow linearly, hyperbolic distances expand exponentially as points move away from the origin. This property allows hyperbolic space to efficiently embed tree-like and hierarchical structures, as the distances between levels of a hierarchy are naturally stretched, ensuring clear separability between different clusters. Moreover, the Poincaré Disc model maintains conformal mapping, meaning local angles and relative positioning are preserved, which helps retain structural integrity when clustering complex, high-dimensional data. These characteristics make the Poincaré Disc metric particularly well-suited for capturing latent relationships and hierarchical patterns in non-Euclidean datasets where traditional Euclidean methods fall short. A filtration step refines connections by focusing on relevant associations, followed by weight updates until convergence. The algorithm culminates in a defuzzification process, where definitive cluster assignments are determined by optimizing the final weight matrix, effectively mapping the continuous membership values to discrete cluster labels.

\noindent
\textbf{Contributions.} Our method enjoys the following advantages: 
\begin{enumerate} 
    \item Our proposed algorithm HypeFCM incorporates the metric of the Poincaré Disc model by embedding the datasets inside the Poincaré Disc and ensuring the subsequent centroid updates are also constrained within the geometry of the Poincaré ball.
    \item Our method introduces a selective filtration process that prunes less significant relationships, optimizes computational efficiency by focusing on the most relevant geometric connections, and improves cluster definition through a targeted weight refinement strategy. 
    \item We present a thorough convergence analysis and discuss the computational complexity analysis of the proposed method. Other necessary proofs and derivations are provided in the Appendix.
\end{enumerate}

\begin{table}[ht]
    \centering
    \caption{Notations.}
    \begin{tabular}{ll}
    \hline
    \textbf{Notation} & \textbf{Description} \\
    \hline
    $\mathbf{X} \in \mathbb{R}^{n \times p}$ & The given dataset \\
    $\mathbf{V} \in \mathbb{R}^{c \times p}$ & The cluster centroid matrix \\
    $\mathbf{W} \in \mathbb{R}^{n \times c}$ & The membership matrix \\
    $\mathbf{x}_i \in \mathbb{R}^{p}$ & The $i$-th sample in $\mathbf{X}$ \\
    $\mathbf{v}_j \in \mathbb{R}^{p}$ & The $j$-th cluster centroid in $\mathbf{V}$ \\
    $\mathbf{1}_c$ & All-ones vector with a length of $c$ \\
    $n$ & Number of samples in $\mathbf{X}$ \\
    $p$ & Dimensionality of $\mathbf{X}$ \\
    $c$ & Number of clusters in $\mathbf{X}$ \\
    $m$ & The value of the fuzziness parameter \\
    $T$ & Maximum number of iterations \\
    $k$ & The Filtration value in the algorithm \\
    $\operatorname{Tr}(\cdot)$ & Denotes the trace of a matrix \\
    $\alpha$ & The curvature parameter of the Poincar\'e Disc model \\
    $\varepsilon$ & The tolerance value \\
    
    \hline
    \end{tabular}
\end{table}

\noindent
The paper is structured as follows: Section \ref{preliminaries} introduces the mathematical preliminaries, Section \ref{related} reviews related works, Section \ref{algo} presents the proposed HypeFCM algorithm in detail, Section \ref{cc} analyzes its computational complexity, while Section \ref{convergence} discusses convergence analysis of our method. Section \ref{experiments} presents experimental results on real and synthetic datasets. Section \ref{discussions} provides detailed discussions of the proposed method, and 
Section \ref{conclusion} concludes the paper.

\section{Preliminaries} \label{preliminaries}
We present the fundamental mathematical preliminaries of hyperbolic geometry required for our work. For a more in-depth exploration, refer to \cite{docarmo}, \cite{spivak}.
\subsection{Hyperbolic space} A Hyperbolic space, represented as $\mathbb{H}^{n}$, is a non-Euclidean space of dimension $n$ which is characterized as a simply-connected Riemannian manifold with a constant negative sectional curvature $-1$. The Killing-Hopf theorem \cite{lee2006riemannian} affirms that any two such Riemannian manifolds are isometrically equivalent. We will briefly discuss the Poincaré Disc Model here. 


\noindent
\textbf{Poincaré Disc Model:} This is a model of hyperbolic space in which all points are inside the unit disc in $\mathbb{R}^{n}$, and geodesics are either diameters or the circular arcs. The metric between two points $\mathbf{X}$ and $\mathbf{Y}\left(\left\|\mathbf{X}\right\|,\left\|\mathbf{Y}\right\|<1\right)$ is defined as,

\begin{equation} 
\begin{aligned}
& d\left(\mathbf{X, Y}\right):=
 \cosh ^{-1}\left(1+\frac{2\left\|\mathbf{X}-\mathbf{Y}\right\|^{2}}{\left(1-\left\|\mathbf{X}\right\|^{2}\right)\left(1-\left\|\mathbf{Y}\right\|^{2}\right)}\right).
\end{aligned}
\end{equation}






\subsection{Gyrovector Spaces}

The framework of gyrovector spaces establishes an elegant non-associative algebraic structure that naturally captures and formalizes the properties of hyperbolic geometry, analogous to the way vector spaces provide the algebraic setting for Euclidean geometry \cite{ungargyrovector}. We denote $\mathbb{D}_{\alpha}^{p}:=\left\{\mathbf{v} \in \mathbb{R}^{p} \mid \alpha\|\mathbf{v}\|^{2}<1\right\}$ taking $\alpha \geq 0$. If $\alpha=0$, then $\mathbb{D}_{\alpha}^{p}=\mathbb{R}^{p}$; if $\alpha>0$, then $\mathbb{D}_{\alpha}^{p}$ is the open ball of radius $1 / \sqrt{\alpha}$. If $\alpha=1$ then we recover the unit ball $\mathbb{D}^{p}$.
 
The operations in Poincar\'e Disc Model are well-defined, computationally efficient as most of the geometric operations have simple closed-form expressions in Poincar\'e Disc Model. In this context, we will provide a brief discussion on m\"obius gyrovector addition and m\"obius scalar multiplication on the Poincar\'e Disc model. Due to the invariance of geometric properties under isometric mappings between the hyperbolic spaces, the fundamental additive and multiplicative algebraic structures can be isomorphically translated across different models of hyperbolic geometry while preserving their essential characteristics \cite{ratcliffehyperbolic}. We have established the isometric equivalence between the Poincaré Disc model and the hyperboloid model by explicitly constructing a mapping between them.
Let \( \mathbb{D}^n = \{ x \in \mathbb{R}^n: \|x\| < 1 \} \) be the Poincaré Disc model, and let the hyperboloid model be,
\[
\mathbb{H}^n = \{ x \in \mathbb{R}^{n+1} : -x_0^2 + x_1^2 + \cdots + x_n^2 = -1,\ x_0 > 0 \}.
\]
Then the mapping:
\[
\phi : \mathbb{D}^n \to \mathbb{H}^n,\quad
x \mapsto \left( \frac{1 + \|x\|^2}{1 - \|x\|^2},\ \frac{2x}{1 - \|x\|^2} \right)
\]
is an isometry. It preserves the hyperbolic distance and induces the same Riemannian metric. Hence, we will be using the Poincar\'e Disc model throughout \cite{hypenn} \cite{numericalstab}.\\
\noindent
\textbf{M\"obius addition.} The M\"obius addition of $\mathbf{v}$ and $\mathbf{w}$ in $\mathbb{D}_{\alpha}^{p}$ is defined as :
\begin{equation}\label{eq: addition}
    \mathbf{v} \oplus_{\alpha} \mathbf{w}:=\frac{\left(1+2 \alpha\langle \mathbf{v}, \mathbf{w}\rangle+\alpha\|\mathbf{w}\|^{2}\right) \mathbf{v}+\left(1-\alpha\|\mathbf{v}\|^{2}\right) \mathbf{w}}{1+2 \alpha\langle \mathbf{v}, \mathbf{w}\rangle+\alpha^{2}\|\mathbf{v}\|^{2}\|\mathbf{w}\|^{2}}.
\end{equation}
In particular, when $\alpha=0$, this conforms with the Euclidean addition of two vectors in $\mathbb{R}^{p}$.
However, it satisfies $\mathbf{v} \oplus_{\alpha} \mathbf{0}=\mathbf{0} \oplus_{\alpha} \mathbf{v}= \mathbf{v}$. Moreover, for any $\mathbf{v}, \mathbf{w} \in \mathbb{D}_{\alpha}^{p}$, we have $(-\mathbf{v}) \oplus_{\alpha} \mathbf{v}= \mathbf{v} \oplus_{\alpha}(-\mathbf{v})=\mathbf{0}$ and $(-\mathbf{v}) \oplus_{\alpha}\left(\mathbf{v} \oplus_{\alpha} \mathbf{w}\right)= \mathbf{w}$.

\noindent
\textbf{M\"obius scalar multiplication.} For $\alpha>0$, the M\"obius scalar multiplication of $\mathbf{v} \in \mathbb{D}_{\alpha}^{p} \backslash\{\boldsymbol{0}\}$ by a real number $\lambda \in \mathbb{R}$ is defined as,
\begin{equation} \label{eq: mult}
    \lambda \otimes_{\alpha} \mathbf{v}:=(1 / \sqrt{\alpha}) \tanh \left(\lambda \tanh ^{-1}(\sqrt{\alpha}\|\mathbf{v}\|)\right) \frac{\mathbf{v}}{\|\mathbf{v}\|}.
\end{equation}
and $\lambda \otimes_{\alpha} \mathbf{0}:=\mathbf{0}$. As the parameter $\alpha$ approaches zero, the expression reverts to conventional Euclidean scalar multiplication: $\lim_{\alpha \rightarrow 0} \lambda \otimes_{\alpha} \mathbf{v}=\lambda \mathbf{v}$. \\
The Hyperbolic Distance function on $\left(\mathbb{D}_{\alpha}^{p}, g^{\alpha}\right)$ is given by,
\begin{equation}\label{eq: hyberbolic}
    d_{hyp}(\mathbf{v}, \mathbf{w})=(2 / \sqrt{\alpha}) \tanh ^{-1}\left(\sqrt{\alpha}\left\|-\mathbf{v} \oplus_{\alpha} \mathbf{w}\right\|\right) .
\end{equation}
\textbf{Riemannian Log-Exp Map.} Given two points $x,y \in \mathbb{D}_{\alpha}^{p}$, the logarithmic map at $x$ applied to $y$ is denoted as:
\begin{equation} \label{eq:logmap}
\log_{\mathbf{x}}(\mathbf{y}) = \frac{2}{\lambda_{\mathbf{x}}} \cdot \frac{\tanh^{-1}\left( \sqrt{\alpha} \cdot \| -\mathbf{x} \oplus_\alpha \mathbf{y} \| \right)}{\| -\mathbf{x} \oplus_\alpha \mathbf{y} \|} \cdot \left( -\mathbf{x} \oplus_\alpha \mathbf{y} \right),
\end{equation}
where ${\lambda_{\mathbf{x}}}= \frac{2}{1- \lambda{\|x\|}^2}$ and $\bigoplus_\alpha$ represents the M\"obius gyrovector addition as Equation \ref{eq: addition}.\\
Given $x\in \mathbb{D}_{\alpha}^{p}$, $v \in T_x(\mathbb{D}_{\alpha}^{p})$, the exponential map is defined as,
\begin{equation} \label{eq:expmap}
\exp_{\mathbf{x}}(\mathbf{v}) = \mathbf{x} \oplus_\alpha \left( \tanh\left( \frac{\lambda_{\mathbf{x}} \cdot \sqrt{\alpha} \cdot \|\mathbf{v}\|}{2} \right) \cdot \frac{\mathbf{v}}{\sqrt{\alpha} \cdot \|\mathbf{v}\|} \right).
\end{equation}
As the parameter $\alpha$ approaches $0$, 
\begin{equation} \label{eq:eucexp}
    \lim_{\alpha \rightarrow 0} \exp_{\mathbf{x}}(\mathbf{v})=\mathbf{(x+v)},
\end{equation}
\begin{equation} \label{eq:euclog}
\lim_{\alpha \rightarrow 0} \log_{\mathbf{x}}(\mathbf{y})=\mathbf{(y-x)},
\end{equation}
which are the exponential and logarithmic maps in Euclidean space.

\section{related works} \label{related}

\subsection{Fuzzy $C$-means (FCM)} The most well-known method for fuzzy clustering is the FCM clustering algorithm\cite{fcm}. Let $\mathbf{X} = \{\mathbf{x}_1, \mathbf{x}_2, \ldots \mathbf{x}_n\}$ be a sample of $n$ observations in $\mathbb{R}^{p}$. Thus, $\mathbf{X}$ is a $(n \times p)$ data matrix, $\mathbf{V}$ is an $(c \times p)$ matrix that represents centroids of the $c$ clusters in $\mathbb{R}^{p}$ and $\mathbf{W}$ is the $(n \times c)$ membership matrix with elements $w_{ij} \in [0,1]$, then the following constrained weighted least square criterion, which is also the objective function for FCM is to be minimized:
    \[ Minimize \hspace{5pt}
    J(\mathbf{X}, \mathbf{W}, \mathbf{C}) = \sum_{i=1}^{n} \sum_{j=1}^{c} (w_{ij})^m \|\mathbf{x}_i - \mathbf{v}_j\|^2,
    \]
    subject to,
    \[
    \sum_{j=1}^{c} w_{ij} = 1, \quad \text{for all } i \in \{1, 2, \dots, n\},
    \]
    \begin{equation} \label{fcm}
    0 < \sum_{i=1}^{n} w_{ij} < n, \quad \text{for all } j \in \{1, 2, \dots, c\}.
    \end{equation} \\
 where, $\mathbf{x}_i$  is the $i^{th}$ row of $\mathbf{X}$, $\mathbf{v}_j$ is the $j^{th}$ row of $\mathbf{V}$ in $\mathbb{R}^{p}$, the centroid of the $j^{th}$ cluster.\\
 The parameter $m \in [2, \infty)$ in \cite{bezdek2013pattern} is the fuzziness parameter. According to \cite{yangsurvey}, the value of $m$ is usually taken as 2.
 Using the Alternative Minimization method, the update formula for the centroid is given by :
\begin{equation}
    \mathbf{v}_j = \frac{\sum_{i=1}^{n} (w_{ij})^{m} \mathbf{x}_i} {\sum_{i=1}^{n} (w_{ij})^{m}}.
\end{equation}
\noindent
The solution for the membership matrix $\mathbf{W}$ is given as:

\begin{equation}
    w_{ij} = \left[ \sum_{k=1}^c \left( \frac{\| \mathbf{x}_i - \mathbf{v}_j \|^2}{\| \mathbf{x}_i - \mathbf{v}_k \|^2} \right)^{\frac{1}{m-1}} \right]^{-1}.
\end{equation}

\subsection{HSFC}
For the clustering problem of the $n$ rows of data matrix $\mathbf{X}$ in $c$ clusters, we can seek for the minimum distance between every $\mathbf{x}_{i}$ and its cluster centroid $\mathbf{v}_{j}$ :

$$
d_{i}^{2}=\min _{\mathbf{v}_{j} \in \mathbf{V}}\left\|\mathbf{x}_{i}-\mathbf{v}_{j}\right\|_{2}^{2}.
$$
Consider the hyperbolic smoothing function $\phi_(y, \tau) = \frac{y+ \sqrt{y^2 +\tau^2}}{2}$, for all $y \in \mathbb{R}, \tau \geq 0$ and the function: \\
$\theta\left(\mathbf{x}_{i}, \mathbf{v}_{j}, \gamma\right)=$ $\sqrt{\sum_{k=1}^{p}\left(x_{ik}-v_{jk}\right)^{2}+\gamma^{2}}$, for $\gamma>0$. Hence, the Minimization Problem of Hyperbolic Smoothing Clustering Method (HSFC) \cite{hyperbolic} is transformed into:

\begin{equation}
\begin{array}{ll}
   \min \hspace{6pt} \sum_{i=1}^{n} d_{i}^{2}, \\ 
   \text {subject to} \sum_{j=1}^{c} \psi_{\sigma}\left(d_{i}-\theta\left(\mathbf{x}_{i}, \mathbf{v}_{j}, \gamma\right), \tau\right) \geq \epsilon,\quad\forall i \in [n]. 
\end{array}
\end{equation}



\subsection{FCSR} The minimization function of the Fuzzy Clustering Guided by Spectral Rotation and Scaling (FCSR) \cite{fcsr} model is defined as:

\begin{equation}
\begin{aligned}
J_{\text{FCSR}}(W, V, R, \Phi) &= \sum_{i=1}^{n} \sum_{k=1}^{c} \| \mathbf{x}_i - \mathbf{v}_k \|_2^2 \, w_{ik}^2 \\
& + \lambda \| W - \Phi F R \|_F^2 \\
\text{s.t.} \quad & W \geq 0, \quad W \mathbf{1} = \mathbf{1}, \\
& R^\top R = I, \\
& \Phi = \text{diag}(\phi_{11}, \phi_{22}, \ldots, \phi_{nn}), \quad \phi_{ii} \geq 0 \\
& \quad (i = 1, 2, \ldots, n),
\end{aligned}
\end{equation}
\noindent
where, $F \in \mathbb{R}^{n \times c}$ is the spectral embedding matrix, $R \in \mathbb{R}^{c \times c}$ is the rotation matrix, $W \in \mathbb{R}^{d \times d'}$ is the projection matrix to a $d'$-dimensional subspace, $\Phi \in \mathbb{R}^{n \times n}$ is the diagonal scaling matrix.

\subsection{EFKM} The objective function of the proposed EFKM \cite{efkm} (Fuzzy $K$-Means Clustering with Discriminative Embedding) is defined as:

\begin{equation}
\min_{\substack{P^\top P = I,\, W \mathbf{1}_c = \mathbf{1}_n,\, W \geq 0,\\ M = [\boldsymbol{\mu}_1, \dots, \boldsymbol{\mu}_c]}} 
\frac{\displaystyle \sum_{i=1}^{n} \sum_{j=1}^{c} w_{ij}^m \left\| P^\top \mathbf{x}_i - \textbf{v}_j \right\|_2^2}
{\operatorname{Tr}(P^\top X X^\top P)}
\end{equation}
\noindent
where, $P \in \mathbb{R}^{d \times d'}$ is the projection matrix to a $d'$-dimensional subspace, $\operatorname{Tr}(\cdot)$ denotes the trace of a matrix, The constraints $W \geq 0$, $W \mathbf{1}_c = \mathbf{1}_n$ ensure that each row of $W$ sums to 1.





\begin{figure*}
    \centering
    \includegraphics[width= 0.9 \linewidth]{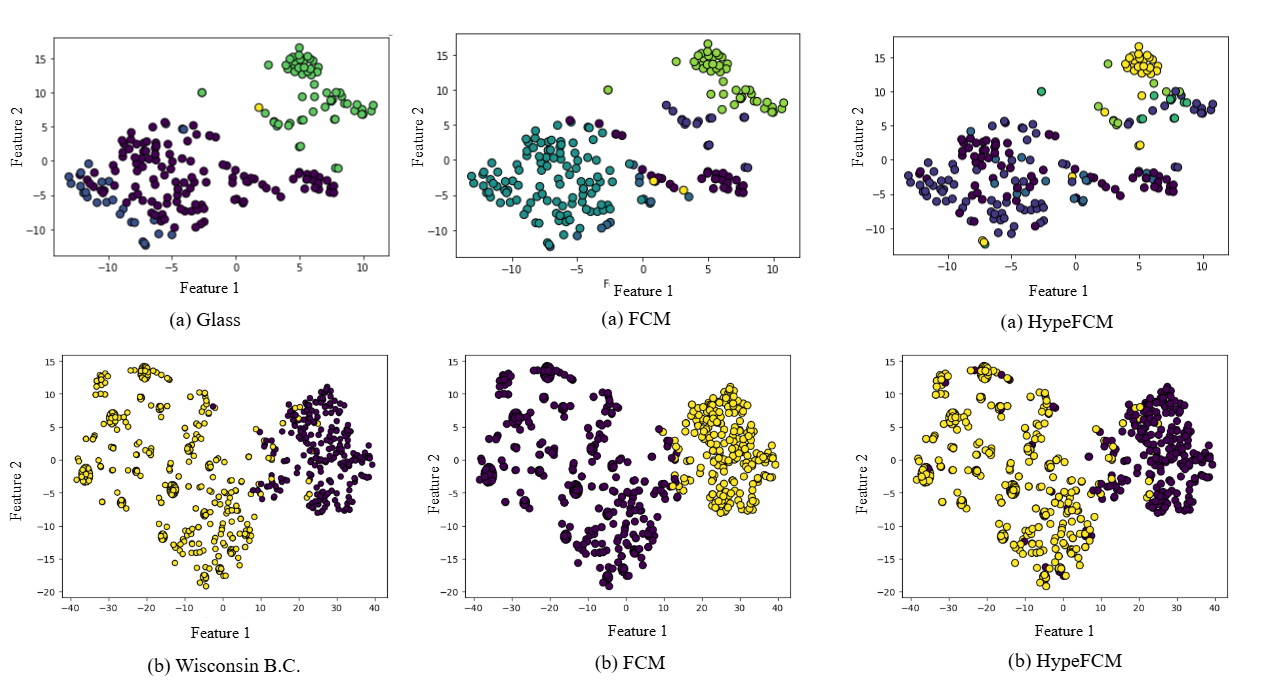}
    \caption{The t-SNE visualization of real datasets (a) Glass, (b) Wisconsin Breast Cancer, respectively, for FCM and HypeFCM clustering methods. }
    \label{fig:real-hfcm}
\end{figure*}

\section{HypeFCM: our proposed method} \label{algo}
In this section, we will present our proposed clustering algorithm, HypeFCM, with intricate details.

\noindent
\textbf{Embedding into a Poincar\'e Disc.}
Our algorithm begins with embedding the dataset \textbf{X}$_{(n\times p)}$ into the hyperbolic space, here the Poincar\'e Disc model of radius $\frac{1}{\sqrt{\alpha}}$. We obtain $\mathbf{X'}= \{\mathbf{x}_1', \mathbf{x}_2', \cdots, \mathbf{x}_n'\} \in \mathbb{D}_{\alpha}^{p}$, where $\mathbb{D}_{\alpha}^{p}:=\left\{\mathbf{x} \in \mathbb{R}^{p} \mid \alpha\|\mathbf{x}\|^{2}<1\right\}$.

\noindent
\textbf{Initializing the Weight Matrix.}
The weight matrix $\textbf{W}= (w_{ij})_{n\times c}$, contains the membership probabilities of the corresponding datapoint of belonging to one of the clusters. Here, the rows of $\textbf{W}$ are initialized as i.i.d. samples from the Dirichlet distribution with all equal probabilities, i.e. 
$${\{w_{ij}\}_{j=1}^{c}\stackrel{i.i.d.}{\sim} Dirichlet (\mathbf{1}_{c}/(\mathbf{1}^T\mathbf{1}))}\quad\forall i \in [n], $$ where, $\mathbf{1}_{c}$ is the vector $(1,1, \cdots,1) \in \mathbb{R}^c$.

\begin{algorithm}[!ht]
        \caption{\footnotesize HypeFCM Algorithm (Filtration-based Hyperbolic Fuzzy $C$-means)}
        \label{main_algo}
        \scriptsize
        {\bfseries Input:} Dataset (\textbf{X}$_{n \times p}$), $k$ = filtration value, $c$ = number of clusters, $m$  = fuzziness parameter, $T$ = maximum number of iterations, $\varepsilon$ = tolerance value.
        \let \oldnoalign \noalign
        \let \noalign \relax
        \let \noalign \oldnoalign
        \begin{algorithmic}[1]
            \STATE Embed the original dataset $\mathbf{X}$ into the Poincar\'e Disc of radius $\frac{1}{\sqrt{\alpha}}$. Obtain $\mathbf{X'}= \{\mathbf{x}_1', \mathbf{x}_2', \cdots, \mathbf{x}_n'\} \in \mathbb{D}_{\alpha}^{p}.$
            
            \STATE Initialize the weight matrix, \textbf{W$^{(0)}_{n \times c} = (w^{(0)}_{ij})_{i\in[n],j\in[c]}$}, with each row being an i.i.d. sample from the Dirichlet distribution with all equal probabilities.
            
            \STATE  Obtain the centroids $\{\mathbf{v}^{(t)}_j\}_{j=1}^{c}$ using the Riemannian log-exp map in Equation \ref{eq:centroid}.

            \STATE  Obtain the distance matrix, \textbf{U$^{(t)}$} by considering distances from the centroids,
            $$
            \mathbf{U^{(t)}_{n \times c}} := u^{(t)}_{ij} := d_{hyp}^2(\mathbf{x}_i',\mathbf{v}^{(t)}_j))_{i\in[n],j\in[c]}.
            $$
            using the squared hyperbolic distance between data points and centroids.

             \STATE  Apply filtration based on the distances of data points from each of the centroids and the rest of the entries of \textbf{U}$^{(t)},$ are assigned to zero to obtain $\mathbf{U}'^{(t)}$. \\

            \STATE  Update the weights using Equation \ref{eq:wt_matrix}.

            \STATE \For{{$t$ from $1$ to $T$}}
            {
                Update $\mathbf{W}^{(t)}$ by Step ($6$)\\
                Update centroids $\mathbf{V}^{(t)} := \{\mathbf{v}^{(t)}_j\}_{j=1}^c$ by Step ($3$)\\
                \If {$||\mathbf{W}^{(t+1)} - \mathbf{W}^{(t)}|| < \varepsilon$}{\textbf{break}}
            }

            \STATE Assign labels by selecting the index of the maximum membership value in each row of the finally updated weight matrix. \\
        {\bfseries Output:} The clusters $Y_1, Y_2, \cdots Y_c.$
        \end{algorithmic}
    
\end{algorithm}

\noindent
\textbf{Computing the Centroids.}
The matrix (\textbf{V}$_{c\times p}$), contains $c$ different centroids in each of its rows. The update expression for the centroids at each iteration is given as:
\begin{equation} 
    \bar{v}_j = \frac{\sum_{i=1}^{n} (w_{ij}^{(t)})^m \cdot \log_{\mathbf{v}_j^{(t-1)}}(\mathbf{x}_i')}{\sum_{i=1}^{n} (w_{ij}^{(t)})^m}
\end{equation}

\begin{equation} \label{eq:centroid}
    \mathbf{v}_j^{(t)} = \exp_{\mathbf{v}_j^{(t-1)}}(\bar{v}_j)
\end{equation}
When limit $\alpha$ tends to $0$, this update formula \ref{eq:centroid} becomes:
$$
\textbf{v}_j^{(t+1)} \xrightarrow{\alpha \to 0} \textbf{v}_j^{(t)} + \frac{\sum_i (w_{ij})^m (\mathbf{x}'_i - \textbf{v}_j^{(t)})}{\sum_i (w_{ij})^m}
= \frac{\sum_i (w_{ij})^m \textbf{x}'_i}{\sum_i (w_{ij})^m},
$$
which conforms to the update formula for the centroid in FCM in Euclidean space, using the Equations \ref{eq:eucexp} and \ref{eq:euclog}.

\noindent
\textbf{Computing the distance matrix and applying filtration.}
Obtain a distance matrix $\mathbf{U}$, $$\mathbf{U^{(t)}_{n \times c}} := u^{(t)}_{ij} := d_{hyp}^2(\mathbf{x}_i', \mathbf{v}^{(t)}_{j})_{i\in[n],j\in[c]},$$
where, $d_{hyp}$ is the hyperbolic distance mentioned in Equation \ref{eq: hyberbolic}.
In each row of the matrix $\mathbf{U}$, the nearest $k$ data points from the respective centroid are kept as it is. The remaining entries are assigned to $0$, that is, given the $t^{th}$ iteration of the centroids, only those data points that are the least distant from the respective centroid are kept and the rest are zero. \\
In essence, we define a relation $\mathbf{R}$ on the Cartesian product of the set of the centroids and the set of the datapoints, where the $i^{th}$ datapoint  and the $j^{th}$ centroid are related such that, 
$$
(i,j) \in \mathbf{R} \iff u_{ij} \leq u_{i'j}, \forall i' \in S \subset [n]  \bigg| n \geq |S|\geq (n-k). 
$$
\noindent
\textbf{Updating the Weight Matrix and Centroids.}
We construct an updated distance matrix, say $\mathbf{U}'_{n \times c}$, where the $(i,j)^{th}$ entry contains the updated distance between $\mathbf{x}_i'$ and $\mathbf{v}_j$ after applying the filtration. The expression for this is given as:
$$u'_{ij} = \Bigl(d_{hyp}^2(\mathbf{x}_{i}'- \mathbf{v}_{j})\mathds{1}\{(i,j) \in \mathbf{R}\}  \Bigr) \quad \forall i \in [n], \forall j \in [c],$$
where $\mathds{1\{.\}}$ denotes the indicator function and $\mathbf{R}$ being the relation as described above. Then the expression for the updation of the weight matrix $\mathbf{W}$ is given as:
\begin{equation}\label{eq:wt_matrix}
    w_{ij}^{(t+1)} = \frac{({u'}_{ij}^{(t)})^{-1/(m-1)}}{\sum_{j = 1}^{c}({u'}_{ij}^{(t)})^{-1/(m-1)}} \quad \forall i \in [n], \forall j \in [c].
\end{equation}

\noindent
\textbf{Assigning labels and obtaining the clusters.}
The cluster labels from the algorithm are then obtained by the defuzzification approach of maximum membership value in the finally updated weight matrix $\mathbf{W}$, assigning the vector containing the labels obtained from the clusters as $\mathbf{z}_{n \times 1}$, which is defined as follows, 
$$
z_i = \arg\max_{j}(w_{ij}) \quad \forall i \in [n],
$$
therefore, the clusters are $\{Y_r\}$, which are defined as $Y_r := \{\mathbf{x}_{i}'\mid z_i = r\} \quad \forall r \in [c]$.

\section{computational complexity analysis} \label{cc}
\noindent
In this section, we are going to present the step-by-step complexity analysis of HypeFCM.

\noindent
\textbf{Embedding into a Poincar\'e Disc.} Embedding the $n$ datapoints into a Poincar\'e Disc requires the time complexity of $\mathcal{O}(n p)$, where $p$ is the dimension of the points.

\noindent
\textbf{Initializing the Weight Matrix.}
Initializing each row of $\mathbf{W}$ from the Dirichlet distribution requires the time complexity of $\mathcal{O}(c)$. For $n$ rows, the time complexity is $\mathcal{O}(n c)$.

\noindent
\textbf{Computing the Centroids.}
For each of the $c$ centroids, computing $(w_{ij})^m$ for all data points requires $\mathcal{O}(n)$. Each centroid update involves weighted Fréchet mean computation using the Riemannian logarithm and exponential maps, which requires the time complexity of $\mathcal{O}(p)$ in each case. Thus, the total time complexity for this step becomes $\mathcal{O}(cnp)$.

\noindent
\textbf{Computing the distance matrix and applying filtration.}
Computing the hyperbolic distance for $n$ points from the $c$ centroids requires $\mathcal{O}(cn)$ many computations. For each of the $n$ rows, filtering the $k$ nearest elements from the respective centroid requires the total time complexity $\mathcal{O}(nc\log k)$.

\noindent
\textbf{Updating the Weight matrix and centroids.}
We update the weight matrix $\mathbf{W}$ using Equation \ref{eq:wt_matrix}. Computing the reciprocals of the entries of $\mathbf{U'}$ and normalizing by the row sums of $\mathbf{U'}$ requires the time complexity $\mathcal{O}(n c)$.

\noindent
\textbf{Iteration Step.}
Performing steps $3$-$6$ for each iteration in Algorithm \ref{main_algo} requires time complexity of $\mathcal{O}(ncp + nc\log k)$. Thus, the total time complexity for $T$ iterations becomes $\mathcal{O}(T(ncp + nc\log k))$.

\noindent
\textbf{Assigning the labels and obtaining the clusters.}
For each of the $n$ rows, finding the maximum weight among the $c$ values requires the time complexity of $\mathcal{O}(nc)$.

\noindent
Therefore, the overall time complexity for our HypeFCM algorithm is $\mathcal{O}(T(ncp + nc\log k))$, equivalent to $\mathcal{O}(Tncp)$, since $ k\ll n$, which is equivalent to the overall time complexity of the FCM algorithm in Euclidean space.

\begin{table*}[h]
\centering
\caption{\centering{Comparison of Clustering Performance across multiple methods, FCM, PCM, P\_SFCM, $K$-means, MiniBatchKmeans, EFKM, UFCER, FCSR, IFKMHC, HSFC, FCPFM, with our proposed HypeFCM on $3$ synthetic datasets along with $7$ Real-World datasets, presented as Means with Standard Deviations. The best and second-best results are highlighted in boldface and underlined, respectively.}}
\label{tab:dataset}
\resizebox{\columnwidth}{!}{\begin{large}
\begin{tabular}{llcccccccccccc}
\toprule
Datasets & Metric & FCM & PCM & P\_SFCM & $K$-means & MiniBatchKmeans & EFKM & UFCER & FCSR & IFKMHC & HSFC & FCPFM & \textbf{HypeFCM (Ours)} \\
\midrule
\multirow{2}{*}{Cure-t1-2000n} 
& ARI & 0.471$\pm$0.030$^\dagger$ & 0.079$\pm$0.024$^\dagger$ & 0.486$\pm$0.021$^\dagger$ & 0.491$\pm$0.014$^\dagger$ & 0.505$\pm$0.019$^\dagger$ & \underline{0.832$\pm$0.011}$^\approx$ & 0.764$\pm$0.012$^\dagger$ & 0.629$\pm$0.015$^\dagger$ & 0.527$\pm$0.022$^\dagger$ & 0.493$\pm$0.020$^\dagger$ & 0.488$\pm$0.015$^\dagger$ & \textbf{0.837} $\pm$ \textbf{0.012} \\ 
& NMI & 0.574$\pm$0.011$^\dagger$ & 0.074$\pm$0.013$^\dagger$ & 0.640$\pm$0.013$^\dagger$ & 0.747$\pm$0.012$^\dagger$ & 0.774$\pm$0.014$^\dagger$ & \underline{0.781$\pm$0.014}$^\approx$ & 0.714$\pm$0.012$^\dagger$ & 0.774$\pm$0.011$^\dagger$ & 0.556$\pm$0.015$^\dagger$ & 0.762$\pm$0.013$^\dagger$ & 0.750$\pm$0.012$^\dagger$ & \textbf{0.794 $\pm$ 0.013} \\
\midrule
\multirow{2}{*}{Cure-t2-4k} 
& ARI & 0.428$\pm$0.011$^\dagger$ & 0.068$\pm$0.010$^\dagger$ & 0.447$\pm$0.013$^\dagger$ & 0.438$\pm$0.012$^\dagger$ & \underline{0.532$\pm$0.010}$^\dagger$ & 0.437$\pm$0.013$^\dagger$ & 0.341$\pm$0.012$^\dagger$ & 0.482$\pm$0.013$^\dagger$ & 0.421$\pm$0.010$^\dagger$ & 0.503$\pm$0.012$^\dagger$ & 0.427$\pm$0.013$^\dagger$ & \textbf{0.581} $\pm$ \textbf{0.023} \\
& NMI & 0.607$\pm$0.011$^\dagger$ & 0.072$\pm$0.011$^\dagger$ & 0.608$\pm$0.012$^\dagger$ & \underline{0.666$\pm$0.013}$^\approx$ & 0.649$\pm$0.012$^\dagger$ & 0.663$\pm$0.011$^\dagger$ & 0.542$\pm$0.012$^\dagger$ & 0.594$\pm$0.012$^\dagger$ & 0.433$\pm$0.012$^\dagger$ & 0.617$\pm$0.013$^\dagger$ & 0.651$\pm$0.014$^\dagger$ & \textbf{0.676} $\pm$ \textbf{0.018} \\
\midrule
\multirow{2}{*}{Smile1} 
& ARI & 0.542$\pm$0.013$^\dagger$ & 0.142$\pm$0.010$^\dagger$ & 0.551$\pm$0.014$^\dagger$ & 0.527$\pm$0.011$^\dagger$ & 0.551$\pm$0.013$^\dagger$ & 0.615$\pm$0.012$^\approx$ & 0.478$\pm$0.011$^\dagger$ & \underline{0.618$\pm$0.013}$^\approx$ & 0.422$\pm$0.012$^\dagger$ & 0.548$\pm$0.012$^\dagger$ & 0.544$\pm$0.017$^\dagger$ & \textbf{0.621} $\pm$ \textbf{0.013} \\
& NMI & 0.606$\pm$0.012$^\dagger$ & 0.325$\pm$0.013$^\dagger$ & 0.612$\pm$0.012$^\dagger$ & 0.559$\pm$0.012$^\dagger$ & 0.607$\pm$0.013$^\dagger$ & 0.677$\pm$0.012$^\dagger$ & 0.534$\pm$0.012$^\dagger$ & \underline{0.686$\pm$0.013}$^\approx$ & 0.559$\pm$0.013$^\dagger$ & 0.607$\pm$0.013$^\dagger$ & 0.608$\pm$0.035$^\dagger$ & \textbf{0.701} $\pm$ \textbf{0.012} \\
\midrule
\multirow{2}{*}{Iris} & ARI & \underline{0.810 $\pm$ 0.012}$^\approx$ & 0.468 $\pm$ 0.071$^\dagger$ & 0.680 $\pm$ 0.058$^\dagger$ & 0.742 $\pm$ 0.023$^\dagger$ & 0.800 $\pm$ 0.015$^\dagger$ & 0.517 $\pm$ 0.015$^\dagger$ & 0.508 $\pm$ 0.015$^\dagger$ & 0.759 $\pm$ 0.015$^\dagger$ & 0.604 $\pm$ 0.015$^\dagger$ & 0.742 $\pm$ 0.021$^\dagger$ & 0.667 $\pm$ 0.046$^\dagger$ & \textbf{0.812} $\pm$ \textbf{0.015} \\
 & NMI & \underline{0.815 $\pm$ 0.015}$^\approx$ & 0.548 $\pm$ 0.065$^\dagger$ & 0.725 $\pm$ 0.051$^\dagger$ & 0.756 $\pm$ 0.027$^\dagger$ & 0.788 $\pm$ 0.019$^\dagger$ & 0.574 $\pm$ 0.015$^\dagger$ & 0.593 $\pm$ 0.015$^\dagger$ & 0.728 $\pm$ 0.015$^\dagger$ & 0.616 $\pm$ 0.015$^\dagger$ & 0.756 $\pm$ 0.025$^\dagger$ & 0.678 $\pm$ 0.055$^\dagger$ & \textbf{0.817} $\pm$ \textbf{0.014} \\
\midrule
\multirow{2}{*}{Glass} & ARI & 0.228 $\pm$ 0.045$^\dagger$ & 0.005 $\pm$ 0.010$^\dagger$ & 0.129 $\pm$ 0.022$^\dagger$ & 0.270 $\pm$ 0.041$^\dagger$ & 0.210 $\pm$ 0.047$^\dagger$ & \underline{0.448 $\pm$ 0.015}$^\dagger$ & 0.168 $\pm$ 0.015$^\dagger$ & 0.158 $\pm$ 0.015$^\dagger$ & 0.113 $\pm$ 0.015$^\dagger$ & 0.264 $\pm$ 0.052$^\dagger$ & 0.176 $\pm$ 0.065$^\dagger$ & \textbf{0.598} $\pm$ \textbf{0.007} \\
 & NMI & 0.334 $\pm$ 0.038$^\dagger$ & 0.035 $\pm$ 0.016$^\dagger$ & 0.156 $\pm$ 0.026$^\dagger$ & 0.428 $\pm$ 0.036$^\dagger$ & 0.356 $\pm$ 0.042$^\dagger$ & \underline{0.491 $\pm$ 0.015}$^\dagger$ & 0.297 $\pm$ 0.015$^\dagger$ & 0.311 $\pm$ 0.015$^\dagger$ & 0.151 $\pm$ 0.015$^\dagger$ & 0.398 $\pm$ 0.047$^\dagger$ & 0.302 $\pm$ 0.028$^\dagger$ & \textbf{0.609} $\pm$ \textbf{0.016} \\
\midrule
\multirow{2}{*}{Ecoli} & ARI & 0.378 $\pm$ 0.041$^\dagger$ & 0.390 $\pm$ 0.047$^\dagger$ & 0.293 $\pm$ 0.028$^\dagger$ & 0.384 $\pm$ 0.039$^\dagger$ & 0.390 $\pm$ 0.038$^\dagger$ & 0.354 $\pm$ 0.015$^\dagger$ & \underline{0.465 $\pm$ 0.015}$^\dagger$ & 0.286 $\pm$ 0.015$^\dagger$ & 0.351 $\pm$ 0.015$^\dagger$ & 0.424 $\pm$ 0.057$^\dagger$ & 0.384 $\pm$ 0.045$^\dagger$ & \textbf{0.518} $\pm$ \textbf{0.006} \\
 & NMI & 0.472 $\pm$ 0.034$^\dagger$ & 0.349 $\pm$ 0.061$^\dagger$ & 0.324 $\pm$ 0.034$^\dagger$ & 0.534 $\pm$ 0.035$^\dagger$ & 0.525 $\pm$ 0.031$^\dagger$ & 0.456 $\pm$ 0.015$^\dagger$ & 0.487 $\pm$ 0.015$^\dagger$ & 0.504 $\pm$ 0.015$^\dagger$ & 0.466 $\pm$ 0.015$^\dagger$ & \underline{0.547 $\pm$ 0.039}$^\dagger$ & 0.576 $\pm$ 0.025$^\dagger$ & \textbf{0.591} $\pm$ \textbf{0.010} \\
\midrule
\multirow{2}{*}{Wine} & ARI & 0.366 $\pm$ 0.032$^\dagger$ & 0.042 $\pm$ 0.019$^\dagger$ & 0.287 $\pm$ 0.015$^\dagger$ & 0.352 $\pm$ 0.042$^\dagger$ & 0.365 $\pm$ 0.036$^\dagger$ & 0.408 $\pm$ 0.015$^\approx$ & 0.351 $\pm$ 0.015$^\dagger$ & 0.391 $\pm$ 0.015$^\dagger$ & 0.295 $\pm$ 0.015$^\dagger$ & 0.375 $\pm$ 0.038$^\dagger$ & \textbf{0.421} $\pm$ \textbf{0.012}$^\approx$ & \underline{0.412 $\pm$ 0.011} \\
 & NMI & 0.425 $\pm$ 0.036$^\dagger$ & 0.089 $\pm$ 0.024$^\dagger$ & 0.320 $\pm$ 0.019$^\dagger$ & 0.423 $\pm$ 0.037$^\dagger$ & 0.430 $\pm$ 0.032$^\dagger$ & \textbf{0.498 $\pm$ 0.015}$^\approx$ & 0.365 $\pm$ 0.015$^\dagger$ & 0.428 $\pm$ 0.015$^\dagger$ & 0.351 $\pm$ 0.015$^\dagger$ & 0.428 $\pm$ 0.034$^\dagger$ & \underline{0.481 $\pm$ 0.016}$^\approx$ & 0.457 $\pm$ 0.012 \\
\midrule
\multirow{2}{*}{Zoo} & ARI & 0.527 $\pm$ 0.025$^\dagger$ & 0.302 $\pm$ 0.054$^\dagger$ & 0.157 $\pm$ 0.067$^\dagger$ & \underline{0.714 $\pm$ 0.022}$^\approx$ & 0.680 $\pm$ 0.044$^\dagger$ & 0.343 $\pm$ 0.015$^\dagger$ & 0.675 $\pm$ 0.015$^\dagger$ & 0.591 $\pm$ 0.015$^\dagger$ & 0.611 $\pm$ 0.015$^\dagger$ & 0.447 $\pm$ 0.062$^\dagger$ & 0.671 $\pm$ 0.035$^\dagger$ & \textbf{0.726 $\pm$ 0.018} \\
 & NMI & 0.543 $\pm$ 0.029$^\dagger$ & 0.457 $\pm$ 0.049$^\dagger$ & 0.376 $\pm$ 0.058$^\dagger$ & \underline{0.782 $\pm$ 0.025}$^\approx$ & 0.747 $\pm$ 0.039$^\dagger$ & 0.427 $\pm$ 0.015$^\dagger$ & 0.711 $\pm$ 0.015$^\dagger$ & 0.615 $\pm$ 0.015$^\dagger$ & 0.664 $\pm$ 0.015$^\dagger$ & 0.482 $\pm$ 0.058$^\dagger$ & 0.723 $\pm$ 0.045$^\dagger$ & \textbf{0.788} $\pm$ \textbf{0.025} \\
\midrule
\multirow{2}{*}{Flights (5k)} & ARI & 0.024 $\pm$ 0.008$^\dagger$ & -0.003 $\pm$ 0.006$^\dagger$ & 0.014 $\pm$ 0.006$^\dagger$ & 0.013 $\pm$ 0.007$^\dagger$ & 0.015 $\pm$ 0.010$^\dagger$ & \underline{0.035 $\pm$ 0.012}$^\approx$ & 0.002 $\pm$ 0.004$^\dagger$ & 0.021 $\pm$ 0.012$^\dagger$ & 0.013 $\pm$ 0.007$^\dagger$ & 0.029 $\pm$ 0.006$^\dagger$ & 0.007 $\pm$ 0.015$^\dagger$ & \textbf{0.042} $\pm$ \textbf{0.014} \\
 & NMI & 0.017 $\pm$ 0.006$^\dagger$ & 0.031 $\pm$ 0.012$^\dagger$ & 0.011 $\pm$ 0.006$^\dagger$ & 0.026 $\pm$ 0.008$^\dagger$ & 0.034 $\pm$ 0.010$^\dagger$ & \underline{0.048 $\pm$ 0.012}$^\approx$ & 0.014 $\pm$ 0.006$^\dagger$ & 0.024 $\pm$ 0.009$^\dagger$ & 0.037 $\pm$ 0.015$^\dagger$ & 0.042 $\pm$ 0.015$^\dagger$ & 0.021 $\pm$ 0.011$^\dagger$ & \textbf{0.050} $\pm$ \textbf{0.016} \\
\midrule
\multirow{2}{*}{Phishing (5k)} & ARI & 0.000 $\pm$ 0.003$^\dagger$ & 0.000 $\pm$ 0.015$^\dagger$ & 0.060 $\pm$ 0.046$^\dagger$ & 0.001 $\pm$ 0.001$^\dagger$ & 0.001 $\pm$ 0.002$^\dagger$ & 0.169 $\pm$ 0.015$^\dagger$ & 0.013 $\pm$ 0.009$^\dagger$ & 0.231 $\pm$ 0.015$^\dagger$ & 0.001 $\pm$ 0.002$^\dagger$ & 0.002 $\pm$ 0.001$^\dagger$ & \underline{0.307 $\pm$ 0.015}$^\approx$ & \textbf{0.316 $\pm$ 0.021} \\
& NMI & 0.010 $\pm$ 0.005$^\dagger$ & 0.003 $\pm$ 0.008$^\dagger$ & 0.069 $\pm$ 0.031$^\dagger$ & 0.001 $\pm$ 0.007$^\dagger$ & 0.003 $\pm$ 0.001$^\dagger$ & 0.137 $\pm$ 0.015$^\dagger$ & 0.035 $\pm$ 0.015$^\dagger$ & 0.179 $\pm$ 0.015$^\dagger$ & 0.012 $\pm$ 0.005$^\dagger$ & 0.005 $\pm$ 0.002$^\dagger$ & \underline{0.274 $\pm$ 0.012}$^\approx$ & \textbf{0.281 $\pm$ 0.018} \\
\midrule
\textbf{Average Rank} & & 6.975 & 10.687 & 8.375 & 6.112 & 5.375 & 4.262 & 7.687 & 6.562 & 8.137 & 5.812 & 4.937 & \textbf{1.457} \\
\bottomrule

\end{tabular}
\end{large}}
\begin{tablenotes}
    \item $^\dagger$ indicates a statistically significant difference between the performances of the corresponding algorithm and HypeFCM.
    \item $^\approx$ indicates the difference between performances of the corresponding algorithm and HypeFCM is not statistically significant.
\end{tablenotes}
\end{table*}

\section{Convergence Analysis of HypeFCM} \label{convergence}
We will introduce the objective function of HypeFCM. Let $\mathbf{X'} = \{\mathbf{x}_1', \mathbf{x}_2', \ldots \mathbf{x}_n'\}$ be a sample of $n$ observations in $\mathbb{D}_{\alpha}^{p}$. $\mathbf{V}$ is an $(c \times p)$ matrix that represents the centroids of the $c$ clusters in $\mathbb{D}_{\alpha}^{p}$ and $\mathbf{W}$ is an $(n \times c)$ membership matrix with elements $w_{ij} \in [0,1]$ such that the following least squares criterion is minimized. In this section, we offer insights into the convergence analysis of HypeFCM.
\noindent
\[ Minimize \hspace{5pt}
    J_m(\mathbf{X'},\mathbf{W},\mathbf{V}) = \sum_{i=1}^n \sum_{j=1}^c (w_{ij})^m d_{hyp}^2(\mathbf{x}_i', \mathbf{v}_j),
\]

    subject to:
    \[
    \sum_{j=1}^c w_{ij} = 1, \forall i \in [n],
    \]
    \[
    w_{ij} \geq 0, \forall i \in [n], j \in [c],
    \]
    \begin{equation} \label{eq:optimisation}
    \mathbf{x}_i', \mathbf{v}_j \in \mathbb{D}_\alpha^{p}, \forall i \in [n], j \in [c]. 
    \end{equation}

\noindent
\textbf{Membership Update.}
The membership weights $w_{ij}$ are updated as:

\begin{equation}
w_{ij} = \frac{d_{hyp}(\mathbf{x}_i', \mathbf{v}_j)^{-2/(m-1)}}{\sum_{l=1}^c d_{hyp}(\mathbf{x}_i', \mathbf{v}_l)^{-2/(m-1)}}.
\end{equation}

\noindent
\textbf{Centroid Update.}
The hyperbolic centroid update is computed using logarithmic and exponential maps on the Poincaré ball.
\begin{equation}
\mathbf{v}_j^{(t+1)} = \exp_{\mathbf{v}_j^{(t)}} \left( \frac{\sum_{i=1}^n (w_{ij})^m \cdot \log_{\mathbf{v}_j^{(t)}}(\mathbf{x}_i')}{\sum_{i=1}^n (w_{ij})^m} \right),
\end{equation}
where, $\log_{\mathbf{v}}(\cdot)$ and $\exp_{\mathbf{v}}(\cdot)$ are the logarithmic and exponential maps as described in \ref{eq:logmap}, \ref{eq:expmap}.
We show mathematically that this centroid updation belongs to the Poincaré ball. The centroids $\{\mathbf{v}^{(t)}_j\}_{j=1}^{c}$ using the Riemannian log-exp map as Equation \ref{eq:centroid}.
Let,
\[
\delta := \tanh\left( \frac{\sqrt{\alpha} \lambda_{\mathbf{v}_j} \|\bar{v}_j\|}{2} \right).
\]
\noindent
The hyperbolic tangent satisfies the following:
    \[
    \tanh(z) \in (0, 1), \quad \text{for } z > 0.
    \]
In our case,
    \[
    z = \frac{\sqrt{\alpha} \lambda_{\mathbf{v}_j} \|\bar{v}_j\|}{2} > 0,
    \]
    since:
    $\alpha > 0$ implies $\sqrt{\alpha} > 0$, $\lambda_{\mathbf{v}_j} = \frac{2}{1 - \alpha \|\mathbf{v}_j\|^2} > 0$ as $\|\mathbf{v}_j\| < \frac{1}{\sqrt{\alpha}}$, $\|\bar{v}_j\| \geq 0$, hence,
    \[
    \delta = \tanh\left( \frac{\sqrt{\alpha} \lambda_{\mathbf{v}_j} \|\bar{v}_j\|}{2} \right) < 1.
    \]
Thus, $\delta < 1 \implies \|\frac{\delta}{\sqrt{\alpha}} \cdot \frac{\bar{v}_j}{\|\bar{v}_j\|}\| < \frac{1}{\sqrt{\alpha}}$. Thus by Lemma \ref{bounded}, $\mathbf{v}_j^{(t)} \in \mathbb{D}_\alpha^p$. This shows that the centroid update belongs to the Poincaré disc.

\begin{lemma}\label{bounded}
   For any $\mathbf{x},\mathbf{y} \in \mathbb{D}_\alpha^{p}$:
$0 \leq d_{hyp}(\mathbf{x},\mathbf{y}) < \infty$. 
\end{lemma}

\noindent
The detailed proof of this lemma is discussed in the Appendix A.  

\begin{figure*}[ht]
    \centering
    \includegraphics[width= \linewidth]{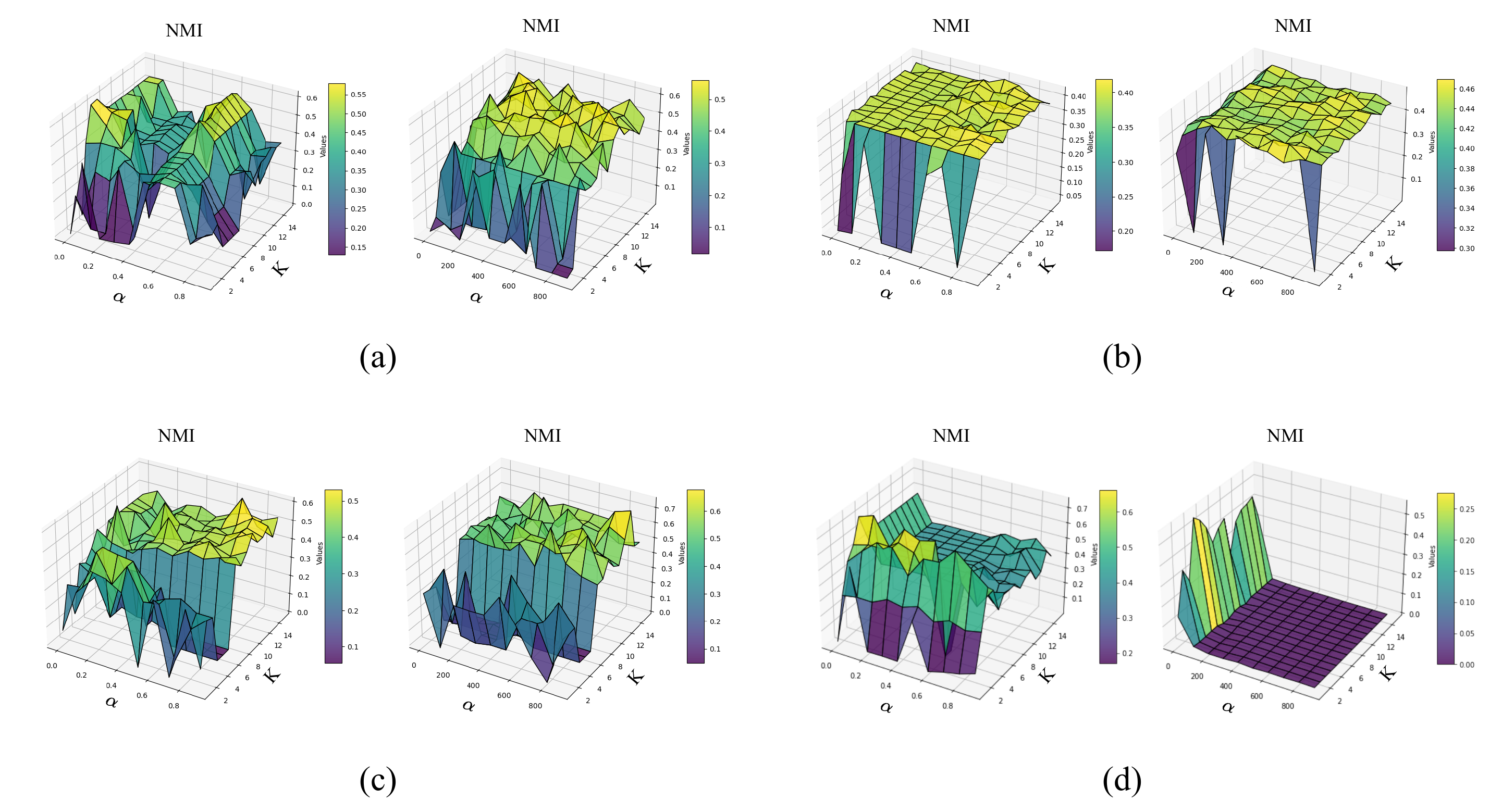}
    \caption{Grid search plots of $2$ real-world datasets (a) Ecoli, (b) Wine and $2$ synthetic datasets (c) Smile1, (d) Cure-t1-2000n, for two ranges of curvatures ($\alpha$), $0$-$1$ and $1$-$1000$ and different filtration ($k$) values against the performance metric, NMI are presented above.}
    \label{fig:grid-plot}
\end{figure*}

\begin{lemma} \label{centroid}
For fixed weights $w_{ij}$, there exists a unique minimizer $\mathbf{v}_j^*$ of the function:
\[
f(\mathbf{v}_j) = \sum_{i=1}^n (w_{ij})^m d_{hyp}^2(\mathbf{x}_i',\mathbf{v}_j).
\]
\end{lemma}
\noindent
The proof of this lemma is discussed in Appendix A.
\begin{theorem}
    Let $J_m^{(t)} = J_m(\mathbf{W}^{(t)},\mathbf{V}^{(t)})$. The sequence $\{J_m^{(t)}\}_{t=0}^{\infty}$ converges to a local minimum.
\end{theorem}
\begin{proof}
We have proved step by step the convergence analysis of the proposed objective function \ref{eq:optimisation}. 

\noindent
\textbf{Step 1: Descent Property of Membership Update.} \\
We aim to show that the HypeFCM algorithm converges to a stationary point of the following objective function defined on the Poincaré ball model $\mathbb{D}_\alpha^p$:
    For fixed $\mathbf{V}^{(t)}$, the weight update minimizes $J_m$ subject to constraints in Equation \ref{eq:optimisation}. \\
For each point $i$, form the Lagrangian
\begin{equation}
    L_i(w_{i1},...,w_{ic},\lambda_i) = \sum_{j=1}^c (w_{ij})^m d_{hyp}^2(\mathbf{x}_i',\mathbf{v}_j^{(t)}) + \lambda_i(\sum_{j=1}^c w_{ij} - 1).
\end{equation}
Taking partial derivatives with respect to $w_{ij}$:

$$
\frac{\partial {L_i}}{\partial {w_{ij}}} = m(w_{ij})^{m-1}d_{hyp}^2(\mathbf{x}_i',\mathbf{v}_j^{(t)}) + \lambda_i = 0.
$$
From the above, the following can be deduced,\\
$$m(w_{ij})^{m-1}d_{hyp}^2(\mathbf{x}_i',\mathbf{v}_j^{(t)}) + \lambda_i = 0$$
$$ \implies w_{ij} = (-\frac{\lambda_i}{m d_{hyp}^2(\mathbf{x}_i',\mathbf{v}_j^{(t)})})^{1/(m-1)}.
$$
Applying the sum Constraint
$\sum_{j=1}^c w_{ij} = 1$:

$$
    \sum_{j=1}^c (-\frac{\lambda_i}{m d_{hyp}^2(\mathbf{x}_i',\mathbf{v}_j^{(t)})})^{1/(m-1)} = 1.
$$
Let $\alpha_i = (-\frac{\lambda_i}{m})^{1/(m-1)}$, then:

$$
\alpha_i \sum_{j=1}^c (d_{hyp}^2(\mathbf{x}_i',\mathbf{v}_j^{(t)}))^{-1/(m-1)} = 1.
$$
Therefore,
$$\alpha_i = \frac{1}{\sum_{j=1}^c (d_{hyp}^2(\mathbf{x}_i',\mathbf{v}_j^{(t)}))^{-1/(m-1)}}.$$
Hence, as the final update, the following is obtained:

$$w_{ij}^{(t+1)} = \frac{d_{hyp}(\mathbf{x}_i',\mathbf{v}_j^{(t)})^{-2/(m-1)}}{\sum_{j=1}^c d_{hyp}(\mathbf{x}_i', \mathbf{v}_j^{(t)})^{-2/(m-1)}}.$$
Therefore,
\[
J_m(\mathbf{W}^{(t+1)}, \mathbf{V}^{(t)}) \leq J_m(\mathbf{W}^{(t)}, \mathbf{V}^{(t)}).
\]
\noindent
\textbf{Step 2: Descent Property of Centroid Update.} \\
For fixed weights $\{w_{ij}^{(t+1)}\}$, the centroid update is given by the Riemannian weighted Fréchet mean:
\begin{equation}
\mathbf{v}_j^{(t+1)} = \exp_{\mathbf{v}_j^{(t)}}\left( \frac{\sum_{i=1}^n (w_{ij}^{(t+1)})^m \log_{\mathbf{v}_j^{(t)}}(\mathbf{x}_i')}{\sum_{i=1}^n (w_{ij}^{(t+1)})^m} \right),
\end{equation}
where $\log_{\mathbf{v}}(\cdot)$ and $\exp_{\mathbf{v}}(\cdot)$ are the logarithmic and exponential maps at $\mathbf{v}$ in the Poincaré Disc model. This update minimizes the cost function concerning $\mathbf{v}_j$ under Riemannian geometry, which conforms with the one in Equation \ref{eq:centroid}.
For fixed $\mathbf{W}^{(t+1)}$, the centroid update minimizes
$$
    J_m:\mathbf{v}_j^{(t+1)} = \argmin_{\mathbf{v}_j \in \mathbb{D}_\alpha^{p}} \sum_{i=1}^n (w_{ij}^{(t+1)})^m d_{hyp}^2(\mathbf{x}_i',\mathbf{v}_j).
$$ 
\[
\because J_m(\mathbf{W}^{(t+1)}, \mathbf{V}^{(t+1)}) \leq J_m(\mathbf{W}^{(t+1)}, \mathbf{V}^{(t)}).
\] 
Therefore,
\begin{equation} \label{eq:monotone}
  J_m^{(t)} \geq J_m(\mathbf{W}^{(t+1)},\mathbf{V}^{(t)}) \geq J_m(\mathbf{W}^{(t+1)},\mathbf{V}^{(t+1)}) = J_m^{(t+1)}.  
\end{equation}
\noindent
\textbf{Step 3: Monotonicity and Convergence.} \\
    $\{J_m^{(t)}\}$ is monotonically decreasing sequence in t by Equation \ref{eq:monotone} and $J_m^{(t)} \geq 0$ by Lemma \ref{bounded}. Therefore, $\{J_m^{(t)}\}$ converges to some $J^* \geq 0$, by Monotone Convergence theorem \cite{rudin}.\\ The filtration step maintains the monotonicity as it only considers $k$ nearest points from each of the centroids at each iteration. Let $\mathbf{U}'^{(t)}$ be the filtered distance matrix, then:
    $$
    \sum_{i=1}^n \sum_{j=1}^c u'^{(t)}_{ij} \leq \sum_{i=1}^n \sum_{j=1}^c u^{(t)}_{ij}.
    $$
Therefore, the unfiltered objective acts as a pointwise majorizer to the one with filtration.
Under these conditions, and by properties of alternating minimization on manifolds \cite{bonnabel}, the sequence of updates converges to a stationary point of the objective function.
\end{proof}

\section{experiments} \label{experiments}
\subsection{Details of the Datasets}
We validate the efficacy of HypeFCM on $6$ synthetic and $12$ real-world datasets. Iris, Glass, Ecoli, Wine, Wisconsin B.C., Phishing URL, Abalone, Glass, Zoo, ORHD (Optical Recognition of Handwritten Digits) datasets are taken from UCI machine learning repository \cite{dua2017uci}; Flights, MNIST datasets are taken from the Kaggle, and Cure-t1-2000n-2D, Cure-t2-4k, Donutcurves, Disk-4000n, Smile1, 3MC are taken from Clustering Benchmark datasets available at\\
https://github.com/deric/clustering-benchmark. Experiments on some of these datasets are shown in Appendix B.
\subsection{Experimental Setup \& Baselines}
The performance of HypeFCM is measured by involving two well-known performance metrics, namely Adjusted Rand Index (ARI) \cite{ari} and Normalized Mutual Information (NMI) \cite{nmi}.
\noindent
The notable base methods, namely the PCM \cite{pcm}, FCM \cite{fcm}, $k$-means \cite{kmeans}, MinibatchKmeans \cite{minibatch}, along with the State-of-the-Art methods like P\_SFCM \cite{psfcm}, EFKM \cite{efkm}, UFCER \cite{ufcer}, FCSR \cite{fcsr}, IFKMHC \cite{ifkmhc}, HSFC \cite{hsfc}, FCPFM \cite{fcpfmc} clustering methods are considered for comparison with our proposed HypeFCM. 


\subsection{Parameter Setting} 
The performance of HypeFCM depends on three key parameters: the curvature parameter ($\alpha$) of the Poincaré Disc, the value of the filtration ($k$), and the fuzziness parameter ($m$).

We have performed two separate $3$D grid search experiments for each dataset with two parameters (curvature ($\alpha$), filtration value ($k$) against the performance metric, NMI with one varying the curvature ($\alpha$) from $0$ to $1$ by increasing the step size of $0.1$ and the other from $1$ to $1000$ by increasing the step size of $100$ in each iteration on two real-world and two synthetic datasets. The grid search plots in Figure \ref{fig:grid-plot} show the performance of HypeFCM on two real-world datasets, Ecoli, Wine, and two synthetic datasets, Smile1, Cure-t1-2000n, respectively.

We also performed two separate $2$D plots for NMI with varying curvature ($\alpha$), for the values of filtration ($k$) varying from $1$ to $15$. We plotted them in two separate plots for $k = 1, 5, 10, 15$ in figure \ref{fig:real-curvature} for two real datasets: Ecoli, Wine, respectively. 
We analyze the performance from the figures and conclude that not much variation in performance is observed with the choice of curvature ($\alpha$) regarding the best choice of filtration value ($k$).

In all our experiments, we have considered the value of the fuzziness parameter ($m$) as $2$.

\begin{figure}[ht]
    \centering
    \includegraphics[width= 0.6\linewidth]{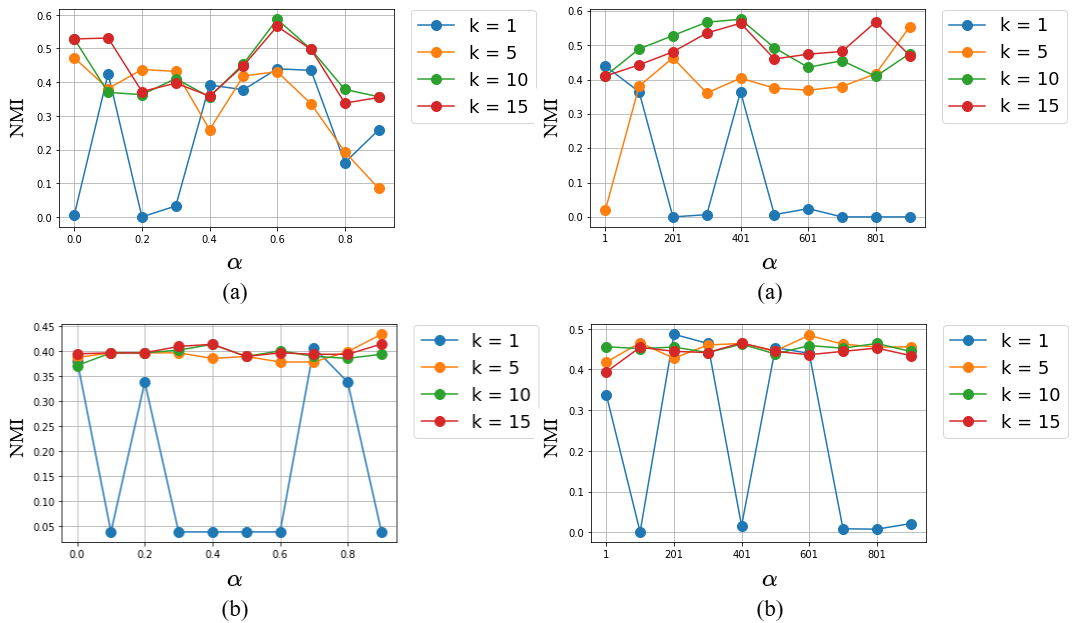}
    \caption{The performance metric NMI vs curvature ($\alpha$) plots of HypeFCM for two real-world datasets (a) Ecoli, (b) Wine respectively.}
    \label{fig:real-curvature}
\end{figure}

\subsection{Experiment on Datasets}
We carried out experiments on $6$ synthetic datasets and $12$ real-world datasets. Performances on some of them are presented in Table \ref{tab:dataset} and the rest of them in the appendix. Our HypeFCM algorithm underscores almost all contenders in terms of ARI and NMI. We have demonstrated the t-SNE visualization \cite{tsne} of HypeFCM in Figure \ref{fig:real-hfcm} for two real-world datasets, Glass, and Wisconsin Breast Cancer. The Average Rank in Table \ref{tab:dataset} is just the average of the ranks of ARI and NMI values obtained from the executions of each method for $15$ times, all initialized with the same seeding at a time. We calculate the $p$-values from the Wilcoxon signed-rank test \cite{wilcoxon} for the methods compared to HypeFCM and represent them in a table in the appendix. Smaller $p$-values ($\leq 0.05$) indicate a statistically significant difference between the performances of the corresponding algorithm and HypeFCM, denoted by $^\dagger$ or else $^\approx$ denotes there is no statistically significant difference. 



\subsection{Ablation Studies}
We conducted an ablation study by varying the key parameter, \textbf{Filtration}. \\
\noindent
\textbf{HypeFCM with and without Filtration.} We conducted two separate experiments, one with filtration and the other without using any filtration on two real-world datasets, Wine, Ecoli, and two synthetic datasets, Cure-t1-2000n and Smile1, respectively. We have performed these experiments by setting the curvature parameter $(\alpha)$ as $1$. Figure \ref{fig:ablation} shows that using filtration improves the performance of HypeFCM compared to the same without filtration.

\begin{figure}
    \centering
    \includegraphics[width= 0.5\linewidth]{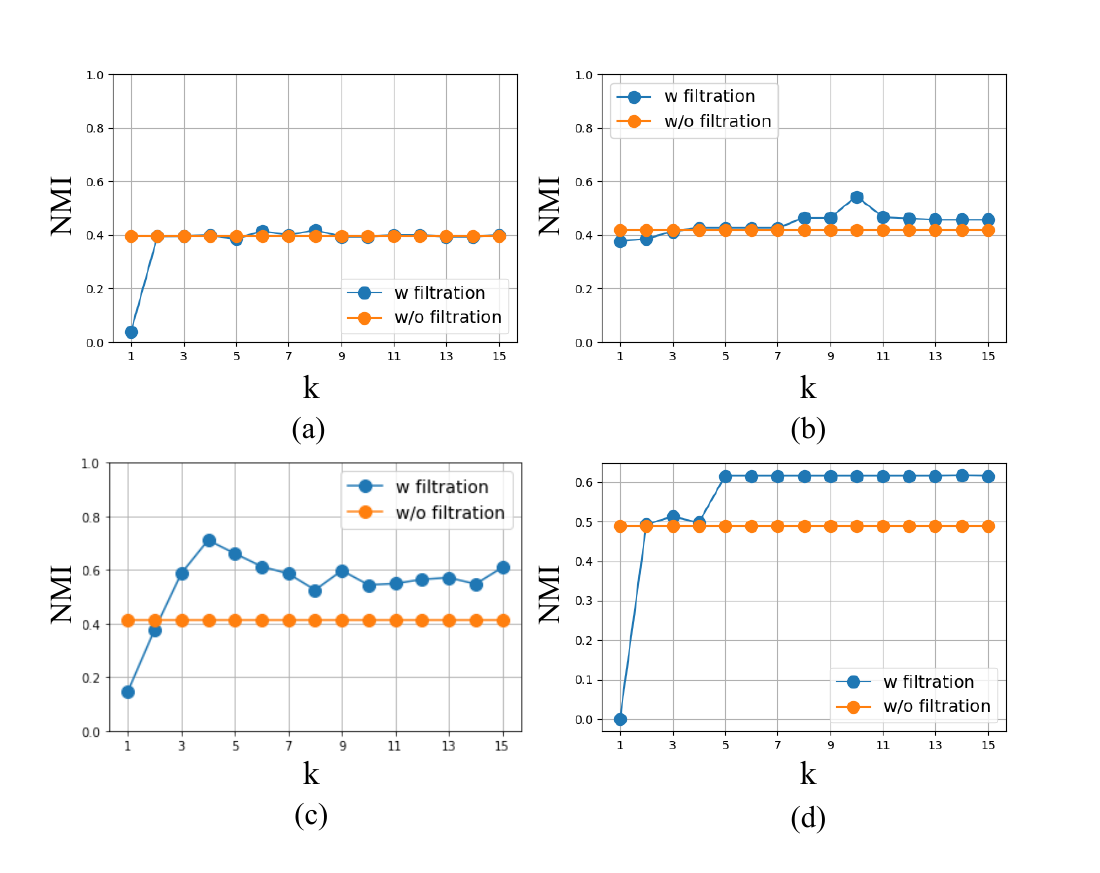}
    \caption{Comparison of Clustering Performance, NMI vs $k$ with and without filtration for two real-world (a) Wine, (b) Ecoli, and two synthetic (c) Cure-t1-2000n-2D, (d) Smile1 datasets respectively.}
    \label{fig:ablation}
\end{figure}

\section{Discussions.} \label{discussions}
The performance results of HypeFCM demonstrate its effectiveness across both synthetic and real-world datasets. The algorithm's superior performance in terms of ARI and NMI metrics against competing methods validates the fundamental hypothesis that hyperbolic geometry can better capture complex hierarchical relationships in clustering tasks. The significant performance of HypeFCM on $12$ real-world benchmark datasets, as visualized through t-SNE for Glass, Wisconsin B.C. datasets, demonstrates how effectively HypeFCM preserves cluster boundaries and captures the inherent hierarchical structures that often characterize real-world datasets. 
Integrating hyperbolic geometry via the Poincaré Disc model represents a significant methodological advancement. The model's exponentially expanding space toward its boundary naturally accommodates hierarchical structures that are prevalent in many real-world datasets but challenging to represent in Euclidean space. This geometric advantage, combined with fuzzy clustering principles, allows HypeFCM to capture subtle hierarchical relationships that traditional fuzzy clustering methods cannot properly identify.

\section{Conclusion \& Future works.} \label{conclusion}
Our proposed HypeFCM marks a significant advancement in handling complex hierarchical data structures by seamlessly integrating hyperbolic geometry with the principles of fuzzy clustering. Our method exploits the unique properties of the Poincaré Disc model, which provides exponentially more space towards its boundary, making it naturally suited for representing hierarchical relationships. Incorporating a weight-based filtering mechanism enhances the algorithm's ability to focus on relevant geometric relationships while pruning less significant connections, thereby improving both computational efficiency and clustering accuracy. The algorithm's iterative optimization process, combining hyperbolic distance calculations with fuzzy membership updates, demonstrates robust performance in capturing complex structures that traditional Euclidean-based methods struggle to represent adequately. The experimental results validate the effectiveness of our approach, showing superior performance in preserving hierarchical relationships while maintaining computational feasibility. While the current results are promising, expanding this framework to handle streaming data and exploring integration with deep learning architectures may further enhance the algorithm's applicability across diverse domains where hierarchical relationships are intrinsic to the data structure.

\bibliographystyle{IEEEtran}
\bibliography{ref}

\newpage
\appendices
\section{Proofs of Lemmas and Theorems from Section \ref{convergence}.}
\textbf{Proof of the Lemma \ref{bounded}.}
\begin{proof}
    The Möbius addition of $\mathbf{x}$ and $\mathbf{y}$ in $\mathbb{D}_\alpha^{p}$ is defined as :
    \begin{equation}
    \mathbf{x} \oplus_{\alpha} \mathbf{y}:=\frac{\left(1+2 \alpha\langle \mathbf{x}, \mathbf{y}\rangle+\alpha\|\mathbf{y}\|^{2}\right) \mathbf{x}+\left(1-\alpha\|\mathbf{x}\|^{2}\right) \mathbf{y}}{1+2 \alpha\langle \mathbf{x}, \mathbf{y}\rangle+\alpha^{2}\|\mathbf{x}\|^{2}\|\mathbf{y}\|^{2}}.
    \end{equation}
    The Hyperbolic Distance function on $\left(\mathbb{D}_\alpha^{p}, g^{\alpha}\right)$ is given by,
    \begin{equation} \label{eq:hyperbolic}
        d_{hyp}(\mathbf{x}, \mathbf{y})=(2 / \sqrt{\alpha}) \tanh ^{-1}\left(\sqrt{\alpha}\left\|-\mathbf{x} \oplus_{\alpha} \mathbf{y}\right\|\right).
    \end{equation} \\
    $d_{hyp} \geq 0$ follows from the properties of $\tanh^{-1}$ and the norm in Equation \ref{eq:hyperbolic}. $d_{hyp}(\mathbf{x},\mathbf{y})$ is finite, since $\|\mathbf{x} \oplus_{\alpha} \mathbf{y}\| < \frac{1}{\sqrt{\alpha}}, \forall \mathbf{x}, \mathbf{y} \in \mathbb{D}_\alpha^{p}$ \& $\|\mathbf{x}\| = \|\mathbf{-x}\|, \forall \mathbf{x}\in \mathbb{D}_\alpha^{p}.$ \\
By the Cauchy-Schwarz inequality, 
\begin{align*}
    &\;|\langle \mathbf{x}, \mathbf{y} \rangle| \leq \|\mathbf{x}\|\|\mathbf{y}\| \\
    \implies& |\alpha  \langle \mathbf{x}, \mathbf{y} \rangle| \leq \alpha \|\mathbf{x}\|\|\mathbf{y}\| \\
    \implies& \big|(1+2\alpha\langle \mathbf{x}, \mathbf{y}\rangle+\alpha\|\mathbf{y}\|^2)\mathbf{x}\big| \\
    &\quad < \frac{1}{\sqrt{\alpha}}\big|(1+2 \alpha\|\mathbf{x}\|\|\mathbf{y}\|+\alpha\|\mathbf{y}\|^2)\big| \\
    \implies& \big|(1+2\alpha\langle \mathbf{x}, \mathbf{y}\rangle+\alpha\|\mathbf{y}\|^2)\mathbf{x} + (1-\alpha\|\mathbf{x}\|^2)\mathbf{y}\big|\\
    &\quad < \big|\frac{1}{\sqrt{\alpha}}(1+2\alpha\|\mathbf{x}\|\|\mathbf{y}\|+\alpha\|\mathbf{y}\|^2)\big|\\
    \implies& \big|(1+2\alpha\langle \mathbf{x}, \mathbf{y}\rangle+\alpha\|\mathbf{y}\|^2)\mathbf{x} \\
    &\quad+ (1-\alpha\|\mathbf{x}\|^2)\mathbf{y}\big| < 3 \frac{1}{\sqrt{\alpha}} \\
    \implies& \bigg|\frac{(1+2\alpha\langle \mathbf{x}, \mathbf{y}\rangle+\alpha\|\mathbf{y}\|^2)\mathbf{x}}{1+2\alpha\langle \mathbf{x}, \mathbf{y}\rangle+\alpha^2\|\mathbf{x}\|^2\|\mathbf{y}\|^2} \\
    &\quad+ \frac{(1-\alpha\|\mathbf{x}\|^2)\mathbf{y}}{1+2 \alpha \langle \mathbf{x}, \mathbf{y}\rangle+\alpha^2\|\mathbf{x}\|^2\|\mathbf{y}\|^2}\bigg| \notag \\
    &\quad < \frac{3(1/\sqrt{\alpha})}{3} \\
    \implies& \|\mathbf{x} \oplus_\alpha \mathbf{y}\| < \frac{1}{\sqrt{\alpha}} . 
\end{align*}
    
\end{proof}

\begin{figure*}
    \centering
    \includegraphics[width= \linewidth]{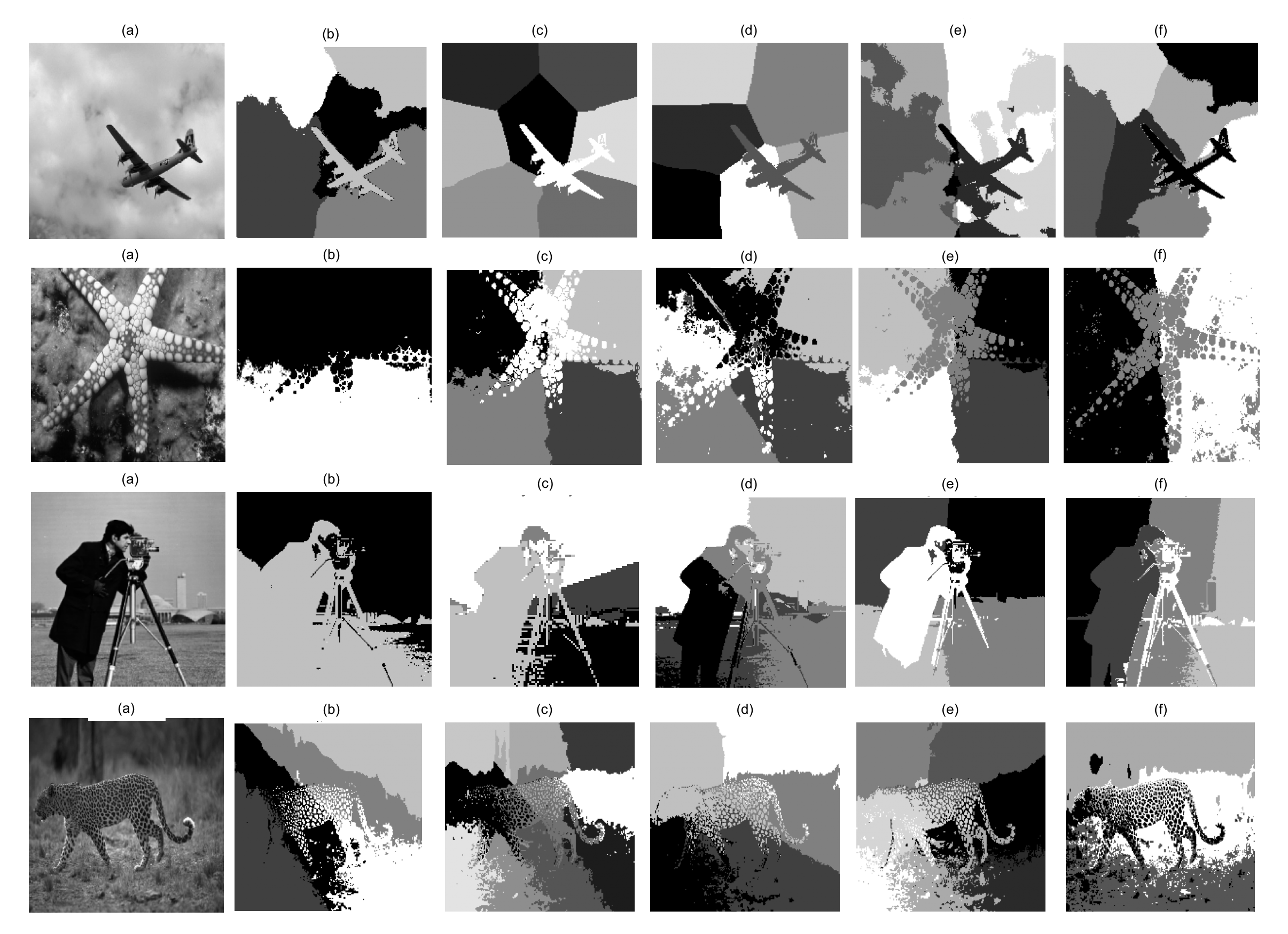}
    \caption{Image segmentation of natural images by different methods (a) Original Image, (b) PCM, (c) FCM, (d) EFKM, (e) HSFC, (f) HypeFCM (ours).}
    \label{fig:image-segment}
\end{figure*}

\begin{figure*}
    \centering
    \includegraphics[width= 0.85\linewidth]{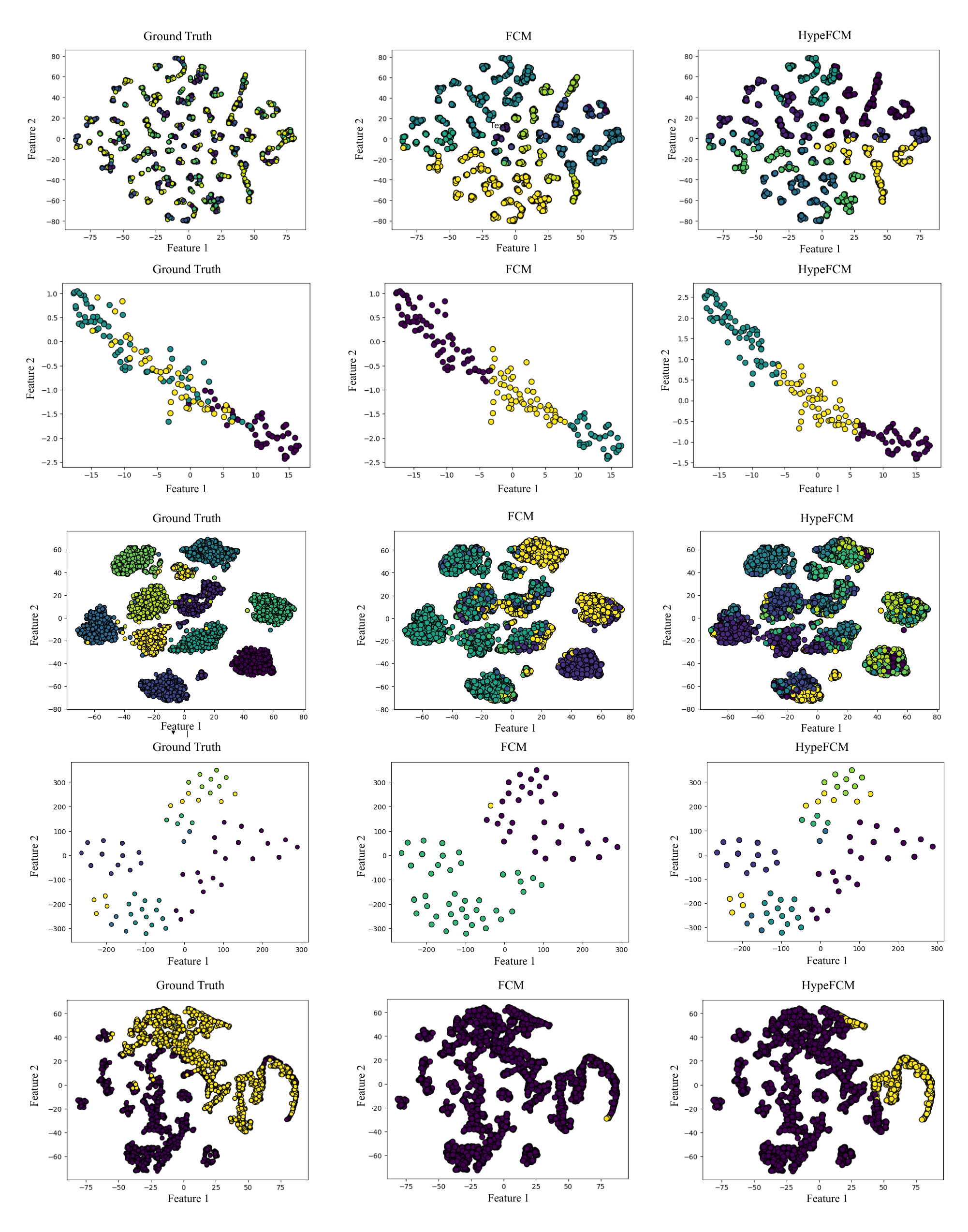}
    \caption{The t-SNE visualization of $5$ real-world benchmark datasets: Flights (5k), Wine, ORHD, Zoo, Phishing URL (5k) for FCM and HypeFCM clustering methods,respectively [from top to bottom].}
    \label{fig:real-tsne}
\end{figure*}

\noindent
\begin{lemma} \label{completeness}
    The Poincar\'e disc model $\mathbb{D}_\alpha^{p}$ with the distance metric $d_{hyp}$ is complete.
\end{lemma} 
\begin{proof}
    This follows from the Hopf-Rinow theorem \cite{hopf} as the hyperbolic space is a geodesically complete Riemannian manifold.
\end{proof}

\begin{table*}[ht]
\centering
\caption{\centering{Comparison of Clustering Performance across multiple Fuzzy contenders, FCM, PCM, P\_SFCM, EFKM, UFCER, FCSR, IFKMHC, HSFC, FCPFM, with our proposed HypeFCM, presented as Means with Standard Deviations. Best and second-best results are highlighted in bold and underlined, respectively.}}
\label{tab:supp_1}
\resizebox{\columnwidth}{!}{\begin{large}
\begin{tabular}{llcccccccccc}
\toprule
Dataset & Metric & FCM & PCM & P\_SFCM & EFKM & UFCER & FCSR & IFKMHC & HSFC & FCPFM & \textbf{HypeFCM (Ours)} \\
\midrule
\multirow{2}{*}{Donutcurves}
& ARI & 0.358$\pm$0.023$^{\dagger}$ & 0.070$\pm$0.010$^{\dagger}$ & 0.498$\pm$0.012$^{\dagger}$ & 0.443$\pm$0.012$^{\dagger}$ & 0.386$\pm$0.012$^{\dagger}$ & \textbf{0.654$\pm$0.013$^{\approx}$} & 0.411$\pm$0.013$^{\dagger}$ & 0.461$\pm$0.013$^{\dagger}$ & 0.428$\pm$0.010$^{\dagger}$ & \underline{0.601 $\pm$ 0.011} \\
& NMI & 0.514$\pm$0.010$^\dagger$ & 0.075$\pm$0.012$^\dagger$ & 0.563$\pm$0.011$^\dagger$ & 0.587$\pm$0.011$^\dagger$ & 0.426$\pm$0.012$^\dagger$ & \textbf{0.681$\pm$0.012}$^\approx$ & 0.524$\pm$0.013$^\dagger$ & 0.634$\pm$0.011$^\dagger$ & 0.539$\pm$0.019$^\dagger$ & \underline{0.667 $\pm$ 0.022} \\\midrule
\multirow{2}{*}{Disk-4000n}
& ARI & 0.001$\pm$0.011$^\dagger$ & 0.016$\pm$0.012$^\dagger$ & 0.001$\pm$0.013$^\dagger$ & -0.004$\pm$0.012$^\dagger$ & 0.001$\pm$0.012$^\dagger$ & 0.029$\pm$0.012$^\dagger$ & 0.001$\pm$0.002$^\dagger$ & \underline{0.101$\pm$0.011}$^\dagger$ & -0.003$\pm$0.082$^\dagger$ & \textbf{0.159 $\pm$ 0.013} \\
& NMI & 0.002$\pm$0.012$^\dagger$ & 0.017$\pm$0.012$^\dagger$ & 0.002$\pm$0.011$^\dagger$ & 0.003$\pm$0.012$^\dagger$ & 0.001$\pm$0.004$^\dagger$ & 0.031$\pm$0.012$^\dagger$ & 0.012$\pm$0.010$^\dagger$ & \underline{0.134$\pm$0.011}$^\dagger$ & 0.002$\pm$0.062$^\dagger$ & \textbf{0.193 $\pm$ 0.012} \\
\midrule
\multirow{2}{*}{Wisconsin B.C.}
& ARI & 0.831$\pm$0.016$^\dagger$ & 0.407$\pm$0.011$^\dagger$ & 0.525$\pm$0.012$^\dagger$ & 0.571$\pm$0.011$^\dagger$ & 0.375$\pm$0.009$^\dagger$ & 0.771$\pm$0.010$^\dagger$ & 0.579$\pm$0.009$^\dagger$ & 0.844$\pm$0.028$^\approx$ & \textbf{0.866$\pm$0.024}$^\approx$ & \underline{0.855 $\pm$ 0.027} \\
& NMI & 0.736$\pm$0.010$^\dagger$ & 0.497$\pm$0.009$^\dagger$ & 0.592$\pm$0.009$^\dagger$ & 0.474$\pm$0.009$^\dagger$ & 0.324$\pm$0.008$^\dagger$ & 0.699$\pm$0.009$^\dagger$ & 0.618$\pm$0.008$^\dagger$ & 0.743$\pm$0.027$^\approx$ & \textbf{0.771$\pm$0.021}$^\approx$ & \underline{0.755 $\pm$ 0.006} \\
\midrule
\multirow{2}{*}{Abalone} 
& ARI & 0.144 $\pm$ 0.062$^\dagger$ & 
0.000 $\pm$ 0.002$^\dagger$ & \textbf{0.161$\pm$0.019}$^\approx$ & 
0.144 $\pm$ 0.015$^\dagger$ & 
0.101 $\pm$ 0.015$^\dagger$ & 
0.003 $\pm$ 0.015$^\dagger$ & 
-0.001 $\pm$ 0.009$^\dagger$ & 
0.103 $\pm$ 0.033$^\dagger$ & 
0.148 $\pm$ 0.010$^\approx$ & 
\underline{0.151 $\pm$ 0.035} \\
& NMI & 0.134 $\pm$ 0.057$^\dagger$ & 
0.004 $\pm$ 0.002$^\dagger$ & \textbf{0.163$\pm$0.027}$^\approx$ & 
0.136 $\pm$ 0.015$^\dagger$ & 
0.123 $\pm$ 0.015$^\dagger$ & 
0.010 $\pm$ 0.015$^\dagger$ & 
0.017 $\pm$ 0.011$^\dagger$ & 
0.097 $\pm$ 0.038$^\dagger$ & 
0.159 $\pm$ 0.009$^\approx$ & 
\underline{0.163 $\pm$ 0.018} \\
\midrule
\multirow{2}{*}{Zoo} 
& ARI & 0.527 $\pm$ 0.025$^\dagger$ & 
0.302 $\pm$ 0.054$^\dagger$ & 
0.157 $\pm$ 0.067$^\dagger$ & 
0.343 $\pm$ 0.015$^\dagger$ & 
\underline{0.675 $\pm$ 0.015}$^\dagger$ & 0.591 $\pm$ 0.015$^\dagger$ & 
0.611 $\pm$ 0.015$^\dagger$ & 
0.447 $\pm$ 0.062$^\dagger$ & 
0.671 $\pm$ 0.035$^\dagger$ & 
\textbf{0.726 $\pm$ 0.018} \\
& NMI & 0.543 $\pm$ 0.029$^\dagger$ & 
0.457 $\pm$ 0.049$^\dagger$ & 
0.376 $\pm$ 0.058$^\dagger$ & 
0.427 $\pm$ 0.015$^\dagger$ & 
0.711 $\pm$ 0.015$^\dagger$ & 
0.615 $\pm$ 0.015$^\dagger$ & 
0.664 $\pm$ 0.015$^\dagger$ & 
0.482 $\pm$ 0.058$^\dagger$ & 
\underline{0.723 $\pm$ 0.045}$^\dagger$ & \textbf{0.788 $\pm$ 0.025} \\
\midrule
\multirow{2}{*}{Flights (2K)}
& ARI & 0.002$\pm$0.008$^\dagger$ & -0.005$\pm$0.006$^\dagger$ & 0.001$\pm$0.006$^\dagger$ & 0.018$\pm$0.012$^\dagger$ & 0.011$\pm$0.004$^\dagger$ & -0.003$\pm$0.012$^\dagger$ & 0.013$\pm$0.007$^\dagger$ & \underline{0.024$\pm$0.006}$^\dagger$ & 0.023 $\pm $ 0.005$^\dagger$ & \textbf{0.046 $\pm$ 0.014} \\
& NMI & 0.012$\pm$0.006$^\dagger$ & 0.020$\pm$0.008$^\dagger$ & 0.011$\pm$0.006$^\dagger$ & 0.031$\pm$0.012$^\dagger$ & 0.018$\pm$0.004$^\dagger$ & 0.028$\pm$0.009$^\dagger$ & 0.041$\pm$0.015$^\dagger$ & \underline{0.042$\pm$0.015}$^\dagger$ & 0.037 $\pm $ 0.009$^\dagger$ & \textbf{0.054 $\pm$ 0.016} \\
\midrule
\multirow{2}{*}{MNIST (5K)}
& ARI & 0.094$\pm$0.017$^\dagger$ & 0.000$\pm$0.004$^\dagger$ & 0.151$\pm$0.014$^\dagger$ & 0.155 $\pm$ 0.019$^\dagger$ & 0.125 $\pm$ 0.012$^\dagger$ & \textbf{0.326$\pm$0.021}$^\approx$ & 0.111$\pm$0.025$^\dagger$ & 0.140$\pm$0.024$^\dagger$ & 0.092 $\pm $ 0.015$^\dagger$ & \underline{0.273 $\pm$ 0.015} \\
& NMI & 0.203$\pm$0.022$^\dagger$ & 0.000$\pm$0.003$^\dagger$ & 0.381$\pm$0.016$^\approx$ & 
0.285 $\pm$ 0.014$^\dagger$ & 
0.243 $\pm$ 0.019$^\dagger$ & \textbf{0.435$\pm$0.024}$^\approx$ & 0.235$\pm$0.030$^\dagger$ & 0.319$\pm$0.026$^\dagger$ &  0.203$\pm$0.055$^\dagger$ & \underline{0.383 $\pm$ 0.011} \\
\midrule
\multirow{2}{*}{ORHD}
& ARI & 0.241$\pm$0.022$^\dagger$ & 0.008$\pm$0.004$^\dagger$ & 0.269$\pm$0.016$^\dagger$ & 0.281 $\pm$ 0.012$^\dagger$ & 0.331 $\pm$ 0.014$^\dagger$ & \textbf{0.421$\pm$0.028}$^\approx$ & 0.231$\pm$0.035$^\dagger$ & \underline{0.349$\pm$0.027}$^\approx$ &  0.231 $\pm $ 0.010$^\dagger$ & 0.335 $\pm$ 0.014 \\
& NMI & 0.398$\pm$0.025$^\dagger$ & 0.080$\pm$0.023$^\dagger$ & 0.421$\pm$0.019$^\dagger$ & 0.411 $\pm$ 0.012$^\dagger$ & 0.394 $\pm$ 0.015$^\dagger$ & 0.455$\pm$0.031$^\approx$ & 0.409$\pm$0.038$^\dagger$ & \underline{0.461$\pm$0.023}$^\approx$ & 0.383 $\pm $ 0.015$^\dagger$ & \textbf{0.466 $\pm$ 0.027} \\
\midrule
\textbf{Average Rank} & & $6.967$ & $9.950$ & $6.783$ & $4.900$ & $6.633$ & $3.300$ & $6.783$ & $3.800$ & $4.133$ & $\mathbf{1.466}$ \\
\bottomrule
\end{tabular}
\end{large}}
\begin{tablenotes}
    \item $^\dagger$ indicates a statistically significant difference between the performances of the corresponding algorithm and HypeFCM.
    \item $^\approx$ indicates the difference between performances of the corresponding algorithm and HypeFCM is not statistically significant.
\end{tablenotes}
\end{table*}

\noindent
\begin{lemma} \label{convexity}
     For any two points $\mathbf{x},\mathbf{y} \in \mathbb{D}_\alpha^{p}$, the function $d_{hyp}(\mathbf{x},\mathbf{y})$ is geodesically convex function in $\mathbb{D}_\alpha^{p}$.
\end{lemma}

\begin{proof}
Let us prove that the hyperbolic distance function $d_{hyp}$ is geodesically convex by showing that for any geodesic $\bm{\gamma}(t)$ in $\mathbb{D}^p_\alpha$, the function $t \mapsto d_{hyp}(\mathbf{x},\bm{\gamma}(t))$ is convex for any fixed point $\mathbf{x}$.

\noindent
Let $\bm{\gamma}: [0,1] \to \mathbb{D}^p_\alpha$ be a geodesic. Define:
\begin{equation}
f(t) = d_{hyp}(x,\bm{\gamma}(t)) = \frac{2}{\sqrt{\alpha}}\tanh^{-1}(\sqrt{\alpha}\|-\mathbf{x} \oplus_\alpha \bm{\gamma}(t)\|).
\end{equation}

\noindent
Computing the first derivative using the chain rule:
\begin{align}
f'(t) &= \frac{2}{\sqrt{\alpha}} \frac{1}{1-\alpha\|-\mathbf{x} \oplus_\alpha \bm{\gamma}(t)\|^2} \frac{d}{dt}\|-\mathbf{x} \oplus_\alpha \bm{\gamma}(t)\|.
\end{align}

\noindent
The second derivative is:
\begin{equation}
    \begin{aligned}
    f''(t) &= \frac{2}{\sqrt{\alpha}}  \Bigg[\frac{2 \alpha \|-\mathbf{x} \oplus_\alpha \bm{\gamma}(t)\|}{(1-\alpha\|-\mathbf{x} \oplus_\alpha \bm{\gamma}(t)\|^2)^2}  \Big(\frac{d}{dt}\|-\mathbf{x} \oplus_\alpha \bm{\gamma}(t)\|\Big)^2 \\
    &\quad + \frac{1}{1-\alpha\|-\mathbf{x} \oplus_\alpha \bm{\gamma}(t)\|^2}  \frac{d^2}{dt^2}\|-\mathbf{x} \oplus_\alpha \bm{\gamma}(t)\|\Bigg].
    \end{aligned}
\end{equation}

\noindent
Now we have : 

1. $\frac{2\alpha \|-\mathbf{x} \oplus_\alpha \bm{\gamma}(t)\|}{(1-\alpha\|-\mathbf{x} \oplus_\alpha \bm{\gamma}(t)\|^2)^2}$ is positive because:
   \begin{itemize}
   \item $\|-\mathbf{x} \oplus_\alpha \bm{\gamma}(t)\| > 0$ for $\mathbf{x} \neq \bm{\gamma}(t),$
   \item $1-\alpha\|-\mathbf{x} \oplus_\alpha \bm{\gamma}(t)\|^2 > 0$ due to the properties of the Poincar\'e disc model.
   \end{itemize}

2. $\Big(\frac{d}{dt}\|-\mathbf{x} \oplus_\alpha \bm{\gamma}(t)\|\Big)^2$ is non-negative.

3. $\frac{d^2}{dt^2}\|-\mathbf{x} \oplus_\alpha \bm{\gamma}(t)\|$ is non-negative along geodesics due to the negative curvature of the hyperbolic space.
\noindent
Since all terms in $f''(t)$ are non-negative, we have:
$$
f''(t) \geq 0.
$$
\noindent
Therefore, $f(t)$ is convex along any geodesic $\bm{\gamma}(t)$, proving that $d_{hyp}$ is geodesically convex.
\end{proof}
\noindent
\textbf{Proof of Lemma \ref{centroid}.}
\begin{proof}
The objective is geodesically convex by Lemma \ref{convexity}. The domain $\mathbb{D}_\alpha^{p}$ is a complete metric space by Lemma \ref{completeness}. By standard optimization theory in metric spaces, specifically the Existence and Uniqueness theorem for convex optimization in complete metric spaces, $f$ attains its minimum on $\mathbb{D}_\alpha^{p}$. The minimizer is unique due to strict geodesic convexity. We denote this unique minimizer as $\mathbf{v}_j^*$.
\end{proof}

    

\begin{table*}[h]
\centering
\caption{\centering{Statistical Significance of Clustering Methods Compared to HypeFCM (ours) via $p$-values for ARI and NMI Metrics.}}
\label{tab:pvalues}
\resizebox{\columnwidth}{!}{
\begin{tabular}{lccccccccccccc}
\toprule
Datasets & Metric & FCM & PCM & P\_SFCM & $K$-means & MiniBatchKMeans & IFKMHC & EFKM & UFCER & FCSR & HSFC & FCPFM \\
\midrule
\multirow{2}{*}{Cure-t1-2000n} & ARI & 6.1000e-05 & 6.1000e-05 & 6.1000e-05 & 6.1000e-05 & 6.1000e-05 & 6.1000e-05 & 2.2811e-01 & 6.1000e-05 & 6.1000e-05 & 6.1000e-05 & 6.1000e-05 \\
 & NMI & 6.1000e-05 & 6.1000e-05 & 6.1000e-05 & 6.1000e-05 & 8.6000e-01 & 6.1000e-05 & 4.1210e-02 & 6.1000e-05 & 6.1000e-05 & 6.1000e-05 & 6.1000e-05 \\
\midrule
\multirow{2}{*}{Cure-t2-4k} & ARI & 6.1000e-05 & 6.1000e-05 & 6.1000e-05 & 6.1000e-05 & 6.1000e-05 & 6.1000e-05 & 6.1000e-05 & 6.1000e-05 & 6.1000e-05 & 6.1000e-05 & 6.1000e-05 \\
 & NMI & 6.1000e-05 & 6.1000e-05 & 6.1000e-05 & 7.1700e-01 & 6.1000e-05 & 6.1000e-05 & 6.1000e-05 & 6.1000e-05 & 6.1000e-05 & 6.1000e-05 & 6.1000e-05 \\
\midrule
\multirow{2}{*}{Donutcurves} & ARI & 6.1000e-05 & 7.2200e-04 & 8.1000e-01 & 6.1000e-05 & 6.1000e-05 & 6.1000e-05 & 6.1000e-05 & 6.1000e-05 & 7.1450e-02 & 6.1000e-05 & 6.1000e-05 \\
 & NMI & 6.1000e-05 & 6.1000e-05 & 4.3100e-01 & 6.1000e-05 & 6.1000e-05 & 6.1000e-05 & 6.1000e-05 & 6.1000e-05 & 6.1000e-01 & 6.1000e-05 & 6.1000e-05 \\
\midrule
\multirow{2}{*}{Smile1} & ARI & 6.1000e-05 & 6.1000e-05 & 6.1000e-05 & 6.1000e-05 & 6.1000e-05 & 6.1000e-05 & 6.1000e-05 & 6.1000e-05 & 3.3100e-01 & 6.1000e-05 & 6.1000e-05 \\
 & NMI & 6.1000e-05 & 6.1000e-05 & 6.1000e-05 & 6.1000e-05 & 6.1000e-05 & 6.1000e-05 & 6.1000e-05 & 6.1000e-05 & 6.3250e-02 & 6.1000e-05  & 6.1000e-05 \\
\midrule
\multirow{2}{*}{Disk-4000n} & ARI & 2.0500e-03 & 5.6000e-05 & 1.0800e-04 & 5.6100e-03 & 5.6100e-03 & 6.1000e-05 & 6.1000e-05 & 6.1000e-05 & 6.1000e-05 & 4.6000e-03 & 6.1000e-05 \\
 & NMI & 6.1000e-05 & 6.1000e-05 & 6.1000e-05 & 6.1000e-05 & 6.1000e-05 & 6.1000e-05 & 6.1000e-05 & 6.1000e-05 & 6.1000e-05 & 6.1000e-05 & 6.1000e-05 \\
\midrule
\multirow{2}{*}{3MC} & ARI & 1.0800e-01 & 6.1000e-05 & 6.1000e-05 & 6.1000e-05 & 6.1000e-05 & 6.1000e-05 & 6.1000e-05 & 6.1000e-05 & 6.1000e-05 & 6.1000e-05 & 6.1000e-05 \\
 & NMI & 1.1100e-01 & 6.1000e-05 & 6.1000e-05 & 6.1000e-05 & 6.1000e-05 & 6.1000e-05 & 6.1000e-05 & 6.1000e-05 & 6.1000e-05 & 6.1000e-05 & 6.1000e-05 \\
\midrule
\multirow{2}{*}{Iris} & ARI & 7.1320e-02 & 6.1000e-05 & 6.1000e-05 & 6.1000e-05 & 6.1000e-01 & 6.1000e-05 & 6.1000e-05 & 6.1000e-05 & 6.1000e-05 & 6.1000e-05 & 6.1000e-05 \\
 & NMI & 8.1571e-02 & 6.1000e-05 & 6.1000e-05 & 6.1000e-05 & 6.1000e-05 & 6.1000e-05 & 6.1000e-05 & 6.1000e-05 & 6.1000e-05 & 6.1000e-05 & 6.1000e-05 \\
\midrule
\multirow{2}{*}{Glass} & ARI & 6.1000e-05 & 2.9000e-02 & 6.1000e-05 & 6.1000e-05 & 6.1000e-05 & 6.1000e-05 & 6.1000e-05 & 6.1000e-05 & 6.1000e-05 & 6.1000e-05 & 6.1000e-05 \\
 & NMI & 6.1000e-05 & 6.1000e-05 & 6.1000e-05 & 6.1000e-05 & 6.1000e-05 & 6.1000e-05 & 6.1000e-05 & 6.1000e-05 & 6.1000e-05 & 6.1000e-05 & 6.1000e-05 \\
 \midrule
\multirow{2}{*}{Ecoli} & ARI & 6.1000e-05 & 6.1000e-05 & 6.1000e-05 & 6.1000e-05 & 6.1000e-05 & 6.1000e-05 & 6.1000e-05 & 6.1000e-05 & 6.1000e-05 & 6.1000e-05 & 6.1000e-05 \\
 & NMI & 6.1000e-05 & 6.1000e-05 & 6.1000e-05 & 6.1000e-05 & 6.1000e-05 & 6.1000e-05 & 6.1000e-05 & 6.1000e-05 & 6.1000e-05 & 6.1000e-05 & 6.1000e-05 \\
\midrule
\multirow{2}{*}{Wine} & ARI & 3.1000e-04 & 6.1000e-05 & 2.9000e-05 & 6.1000e-05 & 8.1000e-04 & 6.1000e-05 & 7.1000e-02 & 6.1000e-05 & 6.1000e-05 & 6.1000e-05 & 1.1110e-01 \\
 & NMI & 4.0100e-05 & 6.1000e-05 & 6.1000e-05 & 6.1000e-05 & 6.1000e-05 & 6.1000e-05 & 1.1000e-01 & 6.1000e-05 & 6.1000e-05 & 2.1000e-02 & 5.1020e-02 \\
\midrule
\multirow{2}{*}{Zoo} & ARI & 6.1000e-05 & 6.1000e-05 & 6.1000e-05 & 1.1120e-01 & 6.1000e-05 & 6.1000e-05 & 6.1000e-05 & 6.1000e-05 & 6.1000e-05 & 6.1000e-05 & 6.1000e-05 \\
 & NMI & 6.1000e-05 & 6.1000e-05 & 6.1000e-05 & 2.1000e-01 & 6.1000e-05 & 6.1000e-05 & 6.1000e-05 & 6.1000e-05 & 6.1000e-05 & 6.1000e-05 & 6.1000e-05 \\
\midrule
\multirow{2}{*}{Abalone} & ARI & 7.1500e-04 & 6.1000e-05 & 1.0800e-01 & 8.8300e-04 & 7.1500e-04 & 6.1000e-05 & 6.1000e-05 & 6.1000e-05 & 6.1000e-05 & 6.1000e-05 & 8.2100e-02 \\
& NMI & 6.1000e-05 & 6.1000e-05 & 1.1000e-01 & 6.1000e-05 & 6.1000e-05 & 6.1000e-05 & 6.1000e-05 & 6.1000e-05 & 6.1000e-05 & 6.1000e-05 & 6.1022e-02 \\
\midrule
\multirow{2}{*}{Phishing (5k)} & ARI & 6.1000e-05 & 6.1000e-05 & 6.1000e-05 & 8.2100e-03 & 2.7400e-03 & 6.1000e-05 & 6.1000e-05 & 6.1000e-05 & 6.1000e-05 & 1.1800e-03 & 1.1000e-01 \\
& NMI & 6.1000e-05 & 6.1000e-05 & 6.1000e-05 & 6.1000e-05 & 6.1000e-05 & 6.1000e-05 & 6.1000e-05 & 6.1000e-05 & 6.1000e-05 & 6.1000e-05 & 8.1201e-02 \\
\midrule
\multirow{2}{*}{Wisconsin B.C.} & ARI & 2.9000e-03 & 6.1000e-05 & 6.1000e-05 & 6.1000e-05 & 6.1000e-05 & 6.1000e-05 & 6.1000e-05 & 6.1000e-05 & 6.1000e-05 & 6.5000e-02 & 5.1230e-01 \\
& NMI & 7.1000e-03 & 6.1000e-05 & 6.1000e-05 & 6.1000e-05 & 6.1000e-05 & 6.1000e-05 & 6.1000e-05 & 6.1000e-05 & 6.1000e-05 & 7.1000e-01 & 3.1025e-01 \\
\midrule
\multirow{2}{*}{Flights (2k)} & ARI & 1.1900e-03 & 7.1000e-04 & 2.9000e-03 & 6.1000e-05 & 6.1000e-05 & 6.1000e-05 & 6.1000e-05 & 6.1000e-05 & 6.1000e-05 & 6.1000e-05 & 6.1000e-05 \\
 & NMI & 6.1000e-05 & 6.1000e-05 & 6.1000e-05 & 6.1000e-05 & 6.1000e-05 & 6.1000e-05 & 6.1000e-05 & 6.1000e-05 & 6.1000e-05 & 6.1000e-05 & 6.1000e-05 \\
\midrule
\multirow{2}{*}{Flights (5k)} & ARI & 6.1000e-05 & 2.0500e-03 & 8.2300e-03 & 6.1000e-05 & 6.1000e-05 & 6.1000e-05 & 8.8100e-02 & 6.1000e-05 & 6.1000e-05 & 6.1000e-05 & 6.1000e-05 \\
 & NMI & 6.1000e-05 & 6.1000e-05 & 6.1000e-05 & 6.1000e-05 & 6.1000e-05 & 6.1000e-05 & 2.1100e-02 & 6.1000e-05 & 6.1000e-05 & 6.1000e-05 & 6.1000e-05 \\
\midrule
\multirow{2}{*}{MNIST (5k)} & ARI & 6.1000e-05 & 6.1000e-05 & 6.1000e-05 & 6.1000e-05 & 6.1000e-05 & 6.1000e-05 & 6.1000e-05 & 6.1000e-05 & 2.1000e-01  & 6.1000e-05 & 6.1000e-05 \\
 & NMI & 6.1000e-05 & 6.1000e-05 & 1.0700e-01 & 6.1000e-05 & 6.1000e-05 & 6.1000e-05 & 6.1000e-05 & 6.1000e-05 & 4.1000e-01  & 6.1000e-05 & 6.1000e-05 \\
\midrule
\multirow{2}{*}{ORHD} & ARI & 6.1000e-05 & 7.1900e-04 & 2.2200e-05 & 6.1000e-05 & 6.1000e-05 & 6.1000e-05 & 6.1000e-05 & 6.1000e-05 & 2.1020e-01 & 1.0800e-01 & 6.1000e-05 \\
 & NMI & 6.1000e-05 & 6.1000e-05 & 7.1000e-04 & 6.1000e-05 & 6.1000e-05 & 6.1000e-05 & 6.1000e-05 & 6.1000e-05 & 6.1560e-02 & 1.1000e-01 & 6.1000e-05 \\
\bottomrule
\end{tabular}}
\end{table*}

\begin{figure*}[ht]
    \centering
    \includegraphics[width= \linewidth]{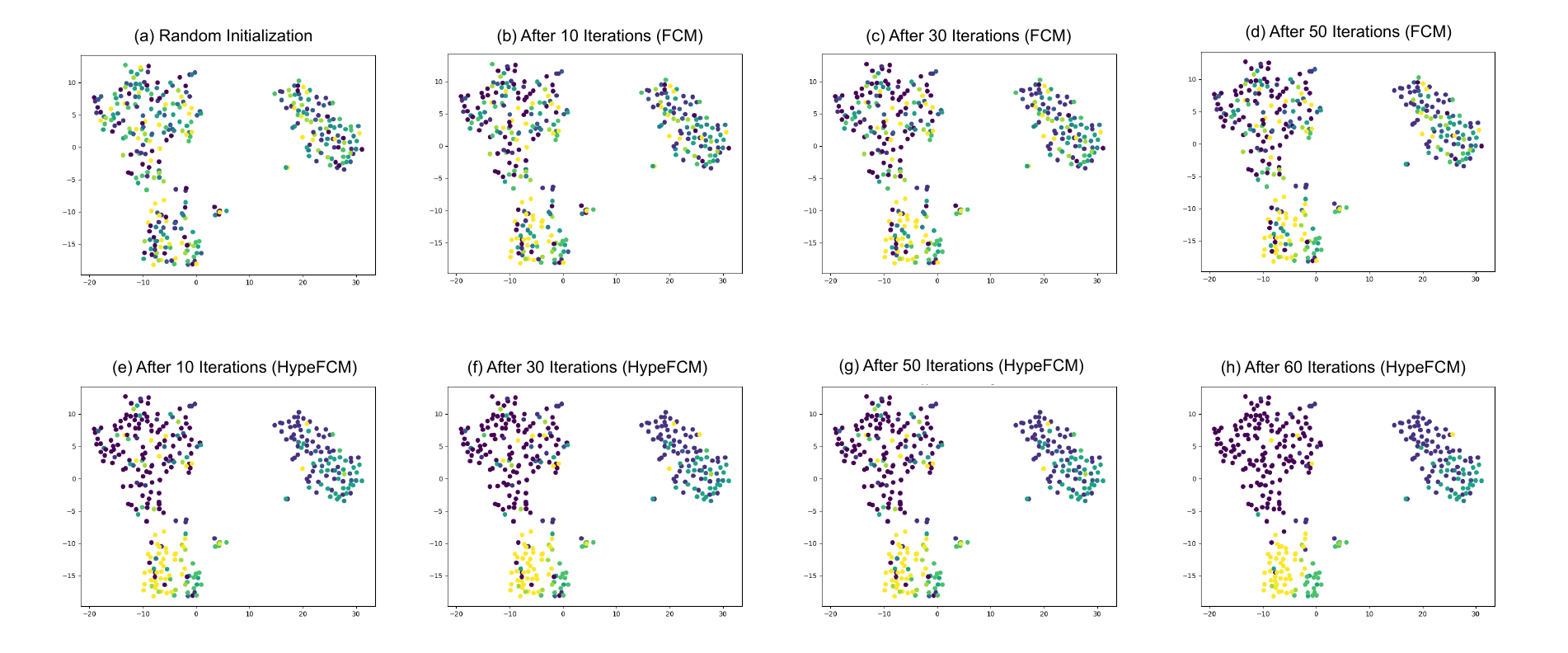}
    \caption{t-SNE plots on Ecoli datasets for different Iterations for FCM and HypeFCM methods.}
    \label{fig:warmup}
\end{figure*}

\section{More Experimentation on Datasets.}
We have compared our proposed method, HypeFCM, with other Fuzzy contenders in our paper, FCM, PCM, P-SFCM, EFKM, UFCER, FCSR, IFKMHC, HSFC, and FCPFM, on additional datasets. 

We have noted the performances of the datasets in Table \ref{tab:supp_1}. The results in Table \ref{tab:supp_1} indicate that our proposed method, HypeFCM, outperforms almost all the fuzzy benchmark contenders across all datasets in terms of ARI \cite{ari} and NMI \cite{nmi} values. The t-SNE \cite{tsne} visualization of HypeFCM in Figure \ref{fig:real-tsne} of $5$ real-world datasets, Flights (5k), Wisconsin Breast Cancer, ORHD, Zoo, Phishing URL (5k), demonstrates how effectively HypeFCM preserves cluster boundaries and captures the inherent hierarchical structures of these datasets. The Average Rank in Table \ref{tab:supp_1} is just the average of the ranks of ARI, and NMI values obtained from the executions of each method $15$ times, all initialized with the same seeding at a time.

In Table \ref{tab:pvalues}, we have shown the statistical significance of the methods compared to our proposed HypeFCM for each dataset via $p$-values for the metrics, ARI and NMI. The Wilcoxon Signed Rank \cite{wilcoxon} statistic is computed by ranking the absolute differences between paired observations of the ARI and NMI values individually for each of the datasets, by contrasting each method $15$ times, all initialized with the same seeding. The final test statistic is the smaller of these sums, and the $p$-value is derived from its distribution to assess whether the method differs significantly from HypeFCM. The lower $p$-values ($p \leq 0.05$) indicate that a statistically significant difference in the performance of the clustering method with our proposed HypeFCM, denoted by $^\dagger$ or else $^\approx$ denotes that there is no significant difference.

We have also tested our method with the other contenders for image segmentation. For the HypeFCM algorithm, we select the fuzzy weighting exponent $m$ = 2, and curvature parameter $\alpha$ = 1. Three images are taken from the BSDS500 dataset, $\#3096, \#12003, \#134052,$ and the Cameraman image dataset from 
https://github.com/antimatter15/cameraman
are used for image segmentation, shown in figure \ref{fig:image-segment}.

\begin{table}[!ht]
    \centering
    \caption{Details of all Datasets used for the Experimentation.}
    \label{tab:description}
    \begin{tabular}{lcccc}
        \toprule
        Datasets & No. of samples & Dimensions & No. of classes\\ 
        \midrule
        Flights & 1048576 & 7 & 16 \\
        PhiUSIIL Phishing URL & 235795 & 54 & 2 \\
        Zoo & 101 & 16 & 7 \\
        Abalone & 4177 & 8 & 2 \\
        Iris & 150 & 4 & 2 \\
        Ecoli & 336 & 7 & 8 \\ 
        Wine & 178 & 13 & 3 \\
        Glass & 214 & 9 & 7\\ 
        Wisconsin Breast Cancer & 699 & 9 & 2 \\
        Abalone & 4177 & 8 & 2 \\ 
        ORHD & 5620 & 64 & 10 \\
        MNIST & 60000 & 784 & 10 \\
        \midrule
        Cure-t1-2000n-2D & 2000 & 2 & 6 \\
        Cure-t2-4k  & 4200 & 2 & 7 \\
        Donutcurves & 1000 & 2 & 4 \\
        Smile1 & 1000 & 2 & 4 \\
        Disk-4000n & 4000 & 2 & 2 \\
        3MC & 400 & 2 & 3 \\        
        \bottomrule
    \end{tabular}
\end{table}

\end{document}